\documentclass[11pt]{article}
\usepackage[a4paper,margin=2.1cm]{geometry}

\usepackage{latexsym}
\usepackage{xspace}
\usepackage{amssymb}
\usepackage{stmaryrd}
\usepackage{hyperref}
\usepackage{url}
\usepackage{graphicx}
\input{xy}
\xyoption{all}
\usepackage{calc}
\usepackage{tikz}
\usepackage{flushend}
\usepackage{color}
\usepackage{cite}
\usepackage[ruled,vlined,linesnumbered]{algorithm2e}
\usepackage{authblk}

\newtheorem{theorem}{Theorem}[section]
\newtheorem{lemma}[theorem]{Lemma}
\newtheorem{proposition}[theorem]{Proposition}
\newtheorem{corollary}[theorem]{Corollary}

\newtheorem{Definition}[theorem]{Definition}
\newtheorem{Example}[theorem]{Example}
\newtheorem{Remark}[theorem]{Remark}

\newenvironment{example}{\begin{Example}\begin{em}}{\end{em}\end{Example}}

\newenvironment{proof}{
	
	\smallskip
	
	\noindent
	{\em Proof.}}{
	
	\smallskip
	
}

\def\eqref#1{(\ref{#1})}

\def\tuple#1{\langle#1\rangle}

\newcommand{\mand}{\sqcap}
\newcommand{\mor}{\sqcup}
\newcommand{\V}{\forall}
\newcommand{\E}{\exists}

\newcommand{\fALCP}{\mbox{$f\!\mathcal{ALC}_\Phi$}\xspace}
\newcommand{\fALCPp}{\mbox{$f\!\mathcal{ALC}_\Phi^0$}\xspace}

\newcommand{\mT}{\mathcal{T}}

\newcommand{\mA}{\mathcal{A}}
\newcommand{\mI}{\mathcal{I}}
\newcommand{\mJ}{\mathcal{J}}
\newcommand{\mIp}{{\mathcal{I}'\!}}

\newcommand{\mZ}{\mathcal{Z}}

\newcommand{\ALC}{$\mathcal{ALC}$\xspace}

\newcommand{\CN}{\mathbf{C}}
\newcommand{\RN}{\mathbf{R}}
\newcommand{\IN}{\mathbf{I}}

\newcommand{\Self}{\mathtt{Self}}

\newcommand{\equivP}{\equiv_\Phi}

\newcommand{\myend}{\mbox{}\hfill{\scriptsize$\blacksquare$}}

\newcommand{\comment}[1]{}
\newcommand{\deleted}[1]{}

\newcommand{\fand}{\varotimes}
\newcommand{\fOr}{\varoplus}
\newcommand{\fneg}{\varominus}
\newcommand{\fto}{\Rightarrow}

\newcommand{\IZOcc}{L}

\newcommand{\cnv}[1]{{#1}^-}
\newcommand{\simP}{\sim_\Phi}

\newcommand{\mybigsqcap}{\bigsqcap}

\newcommand{\FDL}{FDL\xspace}
\newcommand{\FDLs}{FDLs\xspace}

\newcommand{\simPI}{\sim_{\Phi,\mI}}

\newcommand{\mIsimP}{\mI/_{\simP}}
\newcommand{\mIpsimP}{\mIp/_{\simP}}

\newcommand{\Label}{\mathit{Label}}
\newcommand{\SV}{\Sigma_V}
\newcommand{\SE}{\Sigma_E}

\newcommand{\bbP}{\mathbb{P}}
\newcommand{\bbQ}{\mathbb{Q}}
\newcommand{\bbS}{\mathbb{S}}

\newcommand{\itSplit}{\mathit{split}}

\SetKwInput{Input}{Input}
\SetKwInput{Output}{Output}
\SetKwInput{Purpose}{Purpose}
\SetKw{Break}{break}
\SetKw{Continue}{continue}
\SetKw{Return}{return}
\SetKw{New}{new\,}

\newcommand{\CompCB}{\mbox{$\mathsf{ComputeBisimulation}$}\xspace}

\newcommand{\blockEdge}{\mathit{blockEdge}}
\newcommand{\maxKey}{\mathit{maxKey}}

\newcommand{\MinInt}{\mbox{$\mathsf{MinimizeFuzzyInterpretation}$}\xspace}

\newcommand{\email}[1]{\mbox{Email: \url{#1}}}
\def\ramka#1{\begin{center}\fbox{\parbox{\textwidth-0.8em}{#1}}\end{center}}
	
\begin{document}
\sloppy
	
\title{Minimizing Fuzzy Interpretations in Fuzzy Description Logics by Using Crisp Bisimulations}
		
\author{Linh Anh Nguyen}

\affil{\small Institute of Informatics, University of Warsaw, Banacha 2, 02-097 Warsaw, Poland, \email{nguyen@mimuw.edu.pl}}

\affil{\small
	Faculty of Information Technology, Nguyen Tat Thanh University, Ho Chi Minh City, Vietnam
}

\date{}

\maketitle

\begin{abstract}
The problem of minimizing finite fuzzy interpretations in fuzzy description logics (FDLs) is worth studying. For example, the structure of a fuzzy/weighted social network can be treated as a fuzzy interpretation in FDLs, where actors are individuals and actions are roles. Minimizing the structure of a fuzzy/weighted social network makes it more compact, thus making network analysis tasks more efficient. 
In this work, we study the problem of minimizing a finite fuzzy interpretation in a~FDL by using the largest crisp auto-bisimulation.  The considered FDLs use the Baaz projection operator and their semantics is specified using an abstract algebra of fuzzy truth values, which can be any linear and complete residuated lattice. 
We provide an efficient algorithm with a complexity of $O((m \log{l} + n) \log{n})$ for minimizing a given finite fuzzy interpretation $\mI$, where $n$ is the size of the domain of $\mI$, $m$ is number of nonzero instances of atomic roles of $\mI$ and $l$ is the number of different fuzzy values used for instances of atomic roles of~$\mI$. 
We prove that the fuzzy interpretation returned by the algorithm is minimal among the ones that preserve fuzzy TBoxes and ABoxes under certain conditions.

\medskip

\noindent {\em Keywords:} fuzzy description logic, bisimulation, minimization.
\end{abstract}

%===============================================================================

\section{Introduction}
\label{section:intro}

Two states $x$ and $x'$ in two automata are bisimilar if the following three conditions are satisfied: i)~$x$~is an accepting state iff $x'$ is an accepting state; ii)~if there is a $\sigma$-transition (i.e., a transition caused by an action $\sigma$) from~$x$ to a state~$y$, then there must also be a $\sigma$-transition from~$x'$ to a state~$y'$ such that $y$ and $y'$ are bisimilar; iii)~and vice versa. 
Similarly, two possible worlds $x$ and $x'$ in two Kripke models are bisimilar if the following three conditions are met: i)~a~proposition (an atomic formula) is true at $x$ iff it is true at $x'$; ii)~if a possible world $y$ is accessible from $x$ via a relation $\sigma$, then there must exist a possible world $y'$ accessible from $x'$ via the relation $\sigma$ such that $y$ and $y'$ are bisimilar; iii)~and vice versa. 
The bisimilarity relation defined in this way is maximal by its nature. If we define a bisimulation $Z$ in a similar way by replacing a phrase like ``$x$ and $x'$ are bisimilar'' with ``$Z(x,x')$ holds'', then $Z$ is not required to be maximal. Thus, the bisimilarity relation is the largest bisimulation. 

The bisimilarity relation between states of the same automaton (respectively, possible worlds in the same Kripke model), also known as the largest auto-bisimulation, is an equivalence relation. If two states in a finite automaton are bisimilar, we can merge them to reduce the size of the automaton. However, merging each group of bisimilar states in a finite automaton does not result in a minimal automaton. This is because preserving only the recognized language requires a stronger reduction to make a finite automaton minimal in the general case. On the other hand, bisimilarity-based merging can be used to minimize a Kripke model while preserving the set of all formulas that are true in all possible worlds of the model. %or, after deleting all possible worlds that are unreachable from the actual world, the set of all formulas that are true in the actual world.

Description logics are logical formalisms used for representing and reasoning about terminological knowledge~\cite{dlbook}. They are variants of modal logics, where interpretations, individuals, concepts and roles take the place of Kripke models, possible worlds, formulas and accessibility relations, respectively. In essence, a concept is interpreted as a set of individuals and a role is interpreted as a binary relation between individuals. A knowledge base in a description logic usually consists of a TBox (Terminological Box) and an ABox (Assertional Box), where the TBox is a set of axioms about concepts and roles, while the ABox is a set of assertions about individuals.

In practice, data may be vague or fuzzy and we can use fuzzy description logics (\FDLs)~\cite{BobilloCEGPS2015,BorgwardtP17b} instead of description logics. The semantics of the basic concept constructors $\E R.C$ and $\V R.C$ (existential and universal restrictions), which correspond to the modal formulas $\tuple{R}C$ and $[R]C$ in modal logics, are defined in FDLs using the fuzzy operators $\fand$ and $\fto$, respectively (see Section~\ref{sec: FDLs}). Other fuzzy operators such as $\fOr$ and $\fneg$ may also be used to define the semantics of other concept constructors. Initially, FDLs were mostly studied using the Zadeh family of fuzzy operators. Nowadays, FDLs are usually studied using a t-norm $\fand$ on the unit interval $[0,1]$ and its corresponding residuum~$\fto$, where the G\"odel family of fuzzy operators typically provides the simplest case. More general settings for FDLs may use residuated lattices~\cite{Hajek1998,Belohlavek2002,fss/Nguyen22} or so-called algebras of fuzzy truth values~\cite{TFS2022}.

In this work, we consider the problem of minimizing a finite fuzzy interpretation in a FDL by using the largest/greatest auto-bisimulation. There are the following related questions:
\begin{enumerate}
\item What class of FDLs is considered?
\item How can we minimize a finite fuzzy interpretation in the given FDL? How can we do this efficiently? What is the complexity of the algorithm?
\item With regard to what criteria is the resulting fuzzy interpretation minimal? What are logical properties of the minimized fuzzy interpretation?
\end{enumerate}

\subsection{Related Work}

There are two kinds of bisimulations, {\em crisp} and {\em fuzzy}, for fuzzy modal logics and FDLs (see~\cite{Fan15,FSS2020,fss/Nguyen22,TFS2022}). Restricting to the case of using the G\"odel family of fuzzy operators on the unit interval $[0,1]$, logical characterizations of crisp bisimulations between fuzzy Kripke models or fuzzy interpretations in FDLs are expressed in a logical language with involutive negation or the Baaz projection operator. When moving to residuated lattices, the use of the Baaz projection operator is more suitable~\cite{fss/Nguyen22}. 
Crisp bisimulations have been studied for fuzzy transition systems~\cite{CaoCK11,CaoSWC13,DBLP:journals/fss/WuCBD18,DBLP:journals/fss/WuD16}, weighted automata~\cite{DamljanovicCI14}, Heyting-valued modal logics~\cite{EleftheriouKN12}, fuzzy/many-valued modal logics~\cite{Fan15,aml/MartiM18,fuin/Diaconescu20} and \FDLs~\cite{FSS2020,TFS2022}. 
Fuzzy bisimulations have been studied for fuzzy automata~\cite{CiricIDB12,AFB2020,IJA2023}, weighted/fuzzy social networks~\cite{ai/FanL14,IgnjatovicCS15}, fuzzy modal logics~\cite{Fan15,fss/Nguyen22} and \FDLs~\cite{FSS2020}. 

In the preprint~\cite{abs-2010-15671}, together with a coauthor we designed an efficient algorithm with a complexity of $O((m\log{l} + n)\log{n})$ for computing the crisp partition that corresponds to the largest crisp bisimulation of a given finite fuzzy labeled graph, where $n$, $m$ and $l$ are the number of vertices, the number of nonzero edges and the number of different fuzzy degrees of edges of the input graph, respectively. We also provided an algorithm with a complexity of $O((m\log{m} + n)\log{n})$ for a similar problem for the setting with counting successors, which corresponds to the case with qualified number restrictions in description logics and graded modalities in modal logics.
In~\cite{INS2023}, we gave an efficient algorithm with a complexity of $O((m\log{l} + n)\log{n})$ for computing the fuzzy partition that corresponds to the greatest fuzzy auto-bisimulation of a finite fuzzy labeled graph under the G\"odel semantics. We introduced and used a novel notion of a fuzzy partition, which is defined only for finite sets w.r.t.\ the G\"odel t-norm. 

For the background, recall that Hopcroft~\cite{Hopcroft71} gave an efficient algorithm with a complexity of $O(n\log{n})$ for minimizing states in a crisp deterministic finite automaton, Paige and Tarjan~\cite{PaigeT87} gave an efficient algorithm with a complexity of $O((m+n)\log{n})$ for computing the coarsest crisp partition of a finite crisp graph.

\subsection{Motivation}

First, the problem of minimizing finite fuzzy interpretations in FDLs is worth studying. For example, the structure of a fuzzy/weighted social network can be treated as a fuzzy interpretation in FDLs, where actors are individuals and actions are roles. Minimizing the structure of a fuzzy/weighted social network makes it more compact, thus making network analysis tasks more efficient. 
Second, despite the fact that bisimulations have been studied for fuzzy modal logics and fuzzy/weighted social networks~\cite{ai/FanL14,Fan15,IgnjatovicCS15}, we are not aware of works devoted to minimizing fuzzy Kripke models or using the largest/greatest bisimulation to minimize the structure of a fuzzy/weighted social network.
Third, as discussed below, the results of the earlier works~\cite{FSS2020,minimization-by-fBS} on minimizing a finite fuzzy interpretation in a FDL by using the largest/greatest bisimulation have a quite restrictive scope. 

The work \cite{minimization-by-fBS} concerns the problem of minimizing a finite fuzzy interpretation in a FDL under the G\"odel semantics by using the greatest fuzzy auto-bisimulation. The minimization is done by taking the quotient fuzzy interpretation w.r.t.\ the fuzzy equivalence relation that is the greatest fuzzy auto-bisimulation. 
While the minimality w.r.t.\ the preservation of fuzzy TBoxes has been proved for a general case, the minimality w.r.t.\ the preservation of fuzzy ABoxes has been proved only for a very restrictive case where the universal role is used, and it does not hold for the general case. The work \cite{minimization-by-fBS} does not provide any algorithm for computing the quotient fuzzy interpretation. 

Section~6 of the work~\cite{FSS2020} specifies how to minimize a finite fuzzy interpretation in a FDL that uses the G\"odel family of fuzzy operators on the unit interval $[0,1]$ together with involutive negation. The minimization is done by taking the quotient fuzzy interpretation w.r.t.\ the equivalence relation that is the largest crisp auto-bisimulation. The minimality is proved w.r.t.\ certain criteria like the preservation of fuzzy TBoxes and ABoxes. The setting with involutive negation is essential and corresponds to the crispness of the largest auto-bisimulation. The work~\cite{FSS2020} has the following limitations:
\begin{itemize}
\item It does not consider other algebras of fuzzy truth values like the unit interval $[0,1]$ with another t-norm (e.g., product or {\L}ukasiewicz) or complete residuated lattices. Logical properties of minimized fuzzy interpretations are formulated and proved in~\cite{FSS2020} only for FDLs that use the G\"odel family of fuzzy operators on the unit interval $[0,1]$ together with involutive negation. 
\item The quotient fuzzy interpretation w.r.t.\ the equivalence relation that is the largest crisp auto-bisimulation is defined in~\cite{FSS2020} only theoretically. The work~\cite{FSS2020} does not provide any algorithm for computing it. The naive algorithm a usual reader can come up with for the computation is not efficient. 
\end{itemize}

It is desirable to generalize the results of~\cite[Section~6]{FSS2020} for a larger class of FDLs, especially w.r.t.\ the semantics, and to provide an efficient algorithm for the minimization task.

\subsection{Our Contributions}

In this work, by exploiting the results of~\cite{TFS2022} on logical characterizations of crisp bisimulations in FDLs we generalize the results of~\cite[Section~6]{FSS2020} for a large class of FDLs \fALCP with the Baaz projection operator. 
The semantics of the considered \FDLs is specified by using an abstract algebra of fuzzy truth values \mbox{$\tuple{L, \leq, \fand, \fOr, \fto, \fneg, \triangle, 0, 1}$} such that $\tuple{L,\leq}$ is a~linear complete lattice with $0$ and $1$ as the least and greatest elements, $\tuple{L,\fand,1}$ is a commutative monoid, $\fand$ is increasing w.r.t.\ both the arguments, $\fto$ is decreasing w.r.t.\ the first argument and increasing w.r.t.\ the second one, with $(x \fto y) = 1$ iff $x \leq y$, and 
$\triangle$ is the Baaz projection operator (i.e., $\triangle x = 1$ if $x = 1$, and $\triangle x = 0$ otherwise).
In comparison with residuated lattices \cite{Hajek1998,Belohlavek2002}, an algebra of fuzzy truth values is required to be complete and linear, but $\fand$ and $\fto$ are not required to form an adjoint pair.\footnote{Operators $\fand$ and $\fto$ form an adjoint pair if, for every $x, y, z \in L$, $x \fand y \leq z$ iff $x \leq (y \fto z)$.} Besides, such an algebra also contains the binary operator $\fOr$ and the unary operators $\fneg$ and $\triangle$, with no assumptions about~$\fOr$ and~$\fneg$. 
The considered \FDLs \fALCP extend the basic description logic \ALC with fuzzy truth values, the concept constructor \mbox{$C \to D$} and the features specified by $\Phi \subseteq \{\triangle$, $\circ$, $\mor_r$, $*$, $?$, $I$, $U$, $O\}$, with $\triangle \in \Phi$, where $\triangle$ and $O$ denote the Baaz projection operator and nominal, respectively, as concept constructors, while the remaining symbols denote role constructors: sequential composition~($\circ$), union~($\mor_r$), reflexive-transitive closure~($*$), test~($?$), inverse role~($I$) and universal role~($U$). 

We prove that, for $\mI$ being a finite fuzzy interpretation and $\mIsimP$ the quotient fuzzy interpretation of $\mI$ w.r.t.\ the equivalence relation that is the largest crisp $\Phi$-auto-bisimulation of~$\mI$, 
\begin{enumerate}
\item\label{item: HGDHJ 1} if $(1 \fto x) = x$ and ($y \fOr z = 0$ iff $y = z = 0$), for all fuzzy truth values $x$, $y$ and $z$, then $\mIsimP$ is a minimal fuzzy interpretation that validates the same set of fuzzy TBox axioms in \fALCP as $\mI$; 
\item\label{item: HGDHJ 2} if there are individual names and either $U \in \Phi$ or $\mI$ is connected w.r.t.~$\Phi$, then:\footnote{$\mI$ is connected w.r.t.~$\Phi$ if all of its individuals are reachable from the named individuals via roles of \fALCP.}
	\begin{enumerate}
	\item\label{item: HGDHJ 2a} $\mIsimP$ is a minimal fuzzy interpretation $\Phi$-bisimilar to $\mI$, 
	\item\label{item: HGDHJ 2b} $\mIsimP$ is a minimal fuzzy interpretation that validates the same set of fuzzy assertions of the form $C(a) \bowtie p$ in \fALCP as $\mI$, where $\bowtie\ \in \{\geq, >, \leq, <\}$ and $p$ is a fuzzy truth value.
	\end{enumerate}
\end{enumerate}
The premises of the first assertion hold, e.g., when $L = [0,1]$, $\fand$ is a t-norm, $\fto$ is its residuum and $\fOr$ is the corresponding s-norm. 
For the second assertion, it is usual to assume that there are individual names, but the premise $U \in \Phi$ is not a light assumption. If $U \notin \Phi$ and we want to minimize $\mI$ w.r.t.\ the criteria~\ref{item: HGDHJ 2a} and~\ref{item: HGDHJ 2b}, then we can first delete from $\mI$ all individuals unreachable from the named individuals via roles of \fALCP, to obtain a fuzzy interpretation $\mIp$, and then take~$\mIpsimP$.

Apart from the above mentioned results, which are a generalization of the results of~\cite[Section~6]{FSS2020}, we also provide an efficient algorithm with a complexity of $O((m \log{l} + n) \log{n})$ for computing the quotient fuzzy interpretation $\mIsimP$ from $\mI$, where $n$ is the size of the domain of $\mI$, $m$ is number of nonzero instances of atomic roles of $\mI$ and $l$ is the number of different fuzzy values used for instances of atomic roles of~$\mI$. 
It is based on the algorithm developed recently by us and a coauthor in the preprint~\cite{abs-2010-15671} for computing the crisp partition that corresponds to the largest crisp bisimulation of a given finite fuzzy labeled graph. 

\subsection{The Structure of This Work}

The rest of this work is structured as follows. In Section~\ref{section: prel}, we recall definitions and existing results that are needed for this work. In Section~\ref{section: minimization}, we specify how to minimize finite fuzzy interpretations and prove logical properties of the minimized fuzzy interpretations. In Section~\ref{section: computation}, we present our algorithm for the minimization task. We conclude this work in Section~\ref{sec: conc}.

%===============================================================================

\section{Preliminaries}
\label{section: prel}

In this section, we present basic definitions related to algebras of fuzzy truth values, FDLs and crisp bisimulations for FDLs by recalling them from~\cite{TFS2022}. We also present some results of~\cite{TFS2022} or their consequences, which are exploited in the further part of this work. 

\subsection{Algebras of Fuzzy Truth Values}

%\begin{definition}
An {\em algebra of fuzzy truth values} is defined in \cite{TFS2022} as a tuple $\tuple{L, \leq, \fand, \fOr, \fto, \fneg, \triangle, 0, 1}$ such that $\tuple{L,\leq}$ is a~linear complete lattice with $0$ and $1$ as the least and greatest elements, $\fand$, $\fOr$ and $\fto$ are binary operators and $\fneg$ and $\triangle$ are unary operators on $L$ with the following axioms: 
\begin{eqnarray}
&& x \fand y = y \fand x \label{eq: fand-commutativity} \\
&& x \fand (y \fand z) = (x \fand y) \fand z \label{eq: fand-associativity} \\
&& x \fand 1 = x \label{eq: 1-identity} \\
&& \textrm{$x \fand y \leq z \fand t$\ \ if $x \leq z$ and $y \leq t$} \label{eq: fand-monotonicity} \\[1ex]
&& \textrm{$(x \fto y) \leq (z \fto t)$\ \ if $z \leq x$ and $y \leq t$} \label{eq: fto-monotonicity} \\
&& \textrm{$(x \fto y) = 1$ iff $x \leq y$} \label{eq: fto-when-1} \\[1ex]
%
%&& \textrm{$\triangle x =$ (if $x = 1$ then 1 else 0)}. \label{eq: triangle}
&& \triangle x = \left\{\!\!
				 \begin{array}{ll}
				 1 & \textrm{if } x = 1, \\
				 0 & \textrm{otherwise.}
				 \end{array}
				 \right.
\label{eq: triangle}
\end{eqnarray}
%\myend
%\end{definition}

Thus, $\tuple{L,\fand,1}$ is a commutative monoid, $\fand$ is increasing w.r.t.\ both the arguments, $\fto$ is decreasing w.r.t.\ the first argument and increasing w.r.t.\ the second argument and has an important property of residual implications, and $\triangle$ is the Baaz projection operator~\cite{Baaz96,Fan15}. 

When $L$ is the interval $[0,1]$, a binary operator $\fand$ on $L$ with the properties \eqref{eq: fand-commutativity}-\eqref{eq: fand-monotonicity} is called a {\em t-norm}, and if $\fand$ is a left-continuous t-norm, then the binary operator $\fto$ on $L$ defined by \mbox{$(x \fto y) =$} $\sup \{ z \mid z \fand x \leq y\}$ is called the {\em residuum} of~$\fand$. Every residuum has the properties~\eqref{eq: fto-monotonicity} and~\eqref{eq: fto-when-1}. 
The continuous t-norms $\fand$ named after G\"odel, Goguen (also known as product) and {\L}ukasiewicz are listed in Table~\ref{table: t-norms} together with their corresponding s-norms $\fOr$ and residua $\fto$. 

\begin{table}[h]
\center
\begin{tabular}{|c|c|c|c|}
	\hline
	& G\"odel & Product & {\L}ukasiewicz \\
	\hline
	$x \fand y$ & $\min\{x,y\}$ & $x \cdot y$ & $\max\{0, x+y-1\}$ \\
	\hline
	$x \fOr y$ & $\max\{x,y\}$ & $x + y - x \cdot y$ & $\min\{1, x+y\}$ \\
	\hline
	$x \fto y$ 
	& 
	\(
	\left\{
	\!\!\!\begin{array}{ll}
	1 & \textrm{if $x \leq y$} \\ 
	y & \textrm{otherwise}
	\end{array}\!\!\!
	\right.
	\)
	& 
	\(
	\left\{
	\!\!\!\begin{array}{ll}
	1 & \textrm{if $x \leq y$} \\ 
	y/x & \textrm{otherwise}
	\end{array}\!\!\!
	\right.
	\)	
	& $\min\{1, 1 - x + y\}$
	\\ \hline
\end{tabular}
\caption{Popular t-norms, s-norms and residua.\label{table: t-norms}}
\end{table}

From now on, let $\tuple{L, \leq, \fand, \fOr, \fto, \fneg, \triangle, 0, 1}$ be an arbitrary algebra of fuzzy truth values. 
Let the operators $<$, $\geq$ and $>$ on $L$ be defined in the usual way w.r.t.~$\leq$.

\begin{proposition}[{\cite[Proposition~1]{TFS2022}}]\label{prop: JHFJH}
The following properties hold for all \mbox{$x,y \in \IZOcc$}: 
\begin{eqnarray}
&& \textrm{$x \fand y \leq x$\ \ and\ \ $x \fand y \leq y$} \label{eq: req1} \\
&& \textrm{if $x = 0$ or $y = 0$, then $x \fand y = 0$} \label{eq: UYRJH} \\
&& \textrm{if $x \fand y > 0$, then $x > 0$ and $y > 0$} \label{eq: UYRJH2} \\
&& \textrm{$x \fand y = 1$\ \ iff\ \ $x = y = 1$}. \label{eq: req7}
\end{eqnarray}	
\end{proposition}

For a finite set $\Gamma = \{x_1,\ldots,x_n\} \subseteq \IZOcc$ with $n \geq 0$, define 
$\bigotimes\Gamma = x_1 \fand \cdots \fand x_n \fand 1$.

A {\em fuzzy subset} of a set $X$ is a function $f: X \to L$. 
If $f$ is a fuzzy subset of $X$ and $x \in X$, then $f(x)$ means the fuzzy degree in which $x$ belongs to the subset. 
For $\{x_1,\ldots,x_n\} \subseteq X$ and $\{a_1,\ldots,a_n\} \subset L$, we write $\{x_1\!:\!a_1$, \ldots, $x_n\!:\!a_n\}$ to denote the fuzzy subset $f$ of $X$ such that $f(x_i) = a_i$ for $1 \leq i \leq n$ and $f(x) = 0$ for $x \in X \setminus \{x_1,\ldots,x_n\}$. 

%===============================================================================

\subsection{Fuzzy Description Logics}
\label{sec: FDLs}

Let $\RN$ denote a set of {\em role names}, $\CN$ a set of {\em concept names}, and $\IN$ a set of {\em individual names}. Assume that they are pairwise disjoint. Let $\Phi$ be an arbitrary subset of $\{\triangle$, $\circ$, $\mor_r$, $*$, $?$, $I$, $U$, $O\}$ such that $\triangle \in \Phi$. Elements of $\Phi$ denote the additional features included in the \FDL \fALCP defined below. In comparison with~\cite{TFS2022}, we exclude unqualified/qualified number restrictions and the concept constructor $\E r.\Self$. 

A {\em basic role w.r.t.\ $\Phi$} is either a role name $r \in \RN$ or its inverse $\cnv{r}$ in the case $I \in \Phi$. 

{\em Roles} and {\em concepts} of the \FDL \fALCP are defined inductively as follows:
\begin{itemize}
	\item if $r \in \RN$, then $r$ is a role of \fALCP,
	\item if $R$ is a role of \fALCP, then
		\begin{itemize}
		\item if $*$ belongs to $\Phi$, then $R^*$ is a role of \fALCP, 
		\item if $I$ belongs to $\Phi$, then $\cnv{R}$ (the inverse of $R$) is a role of \fALCP, 
		\end{itemize}
	\item if $R$ and $S$ are roles of \fALCP, then 
		\begin{itemize}
		\item if $\circ$ belongs to $\Phi$, then $R \circ S$ is a role of \fALCP, 
		\item if $\mor_r$ belongs to $\Phi$, then $R \mor S$ is a role of \fALCP, 
		\end{itemize}
	\item if $C$ is a concept of \fALCP and $?$ belongs to $\Phi$, then $C?$ is a role of \fALCP, 
	\item if $U \in \Phi$, then $U$ (the universal role) is a role of \fALCP,   
	
	%\medskip
	
	\item if $p \in \IZOcc$, then $p$ is a concept of \fALCP,
	\item if $A \in \CN$, then $A$ is a concept of \fALCP,
	\item if $C$ and $D$ are concepts of \fALCP and $R$ is a role of \fALCP, then $\triangle C$, $\lnot C$, $C \mand D$, $C \mor D$, $C \to D$, $\V R.C$ and $\E R.C$ are concepts of \fALCP,  
	\item if $O \in \Phi$ and $a \in \IN$, then $\{a\}$ is a concept of \fALCP.
\end{itemize}

We use letters $r$ and $s$ to denote role names, $R$ and $S$ to denote roles, $A$ and $B$ to denote concept names, $C$ and $D$ to denote concepts, and $a$ and $b$ to denote individual names.

Given a finite set $\Gamma = \{C_1,\ldots,C_n\}$ of concepts, we define $\mybigsqcap\Gamma = C_1 \mand \ldots \mand C_n \mand 1$.

\comment{By \fALCPp we denote the largest sublanguage of \fALCP that disallows the role constructors $R \mor S$, $R \circ S$, $R^*$, $(C?)$ and the concept constructors $\lnot C$, $C \mor D$, $\V R.C$, $<\!n\,R$, $<\!n\,R.C$, and uses $\to$ only in the form $p \to C$ or $C \to p$, with $p \in \IZOcc$.} 

A {\em fuzzy interpretation} is a pair $\mI = \langle \Delta^\mI, \cdot^\mI \rangle$, where $\Delta^\mI$ is a~non-empty set, called the {\em domain}, and $\cdot^\mI$ is the {\em interpretation function}, which maps each individual name $a$ to an element $a^\mI \in \Delta^\mI$, each concept name $A$ to a function $A^\mI : \Delta^\mI \to \IZOcc$, and each role name $r$ to a function \mbox{$r^\mI : \Delta^\mI \times \Delta^\mI \to \IZOcc$}. It is {\em finite} if $\Delta^\mI$, $\CN$, $\RN$ and $\IN$ are finite. 
%
%The interpretation function~$\cdot^\mI$ is extended to complex roles and concepts as follows, where $x,y \in \Delta^\mI$:
The interpretation function~$\cdot^\mI$ is extended to complex roles and concepts as shown in Figure~\ref{fig: HDHGS}, for $x,y \in \Delta^\mI$.

\begin{figure}
\ramka{
\[
\begin{array}{rcl}
U^\mI(x,y) & \!\!=\!\! & 1 \\
(R^-)^\mI(x,y) & \!\!=\!\! & R^\mI(y,x) \\
(R \mor S)^\mI(x,y) & \!\!=\!\! & \max\{R^\mI(x,y),S^\mI(x,y)\} \\
(R \circ S)^\mI(x,y) & \!\!=\!\! & \sup\{R^\mI(x,z) \fand S^\mI(z,y) \mid z \in \Delta^\mI \} \\
(R^*)^\mI(x,y) & \!\!=\!\! & \sup \{\textstyle\bigotimes\{R^\mI(x_i,x_{i+1}) \mid 0 \leq i < n\} \mid\\ 
			   & & \qquad n \geq 0,\ x_0,\ldots,x_n \in \Delta^\mI, x_0 = x,\ x_n = y\} \\
(C?)^\mI(x,y) & \!\!=\!\! & \textrm{(if $x \neq y$ then 0 else $C^\mI(x)$)}\\[1ex]
p^\mI(x) & \!\!=\!\! & p \\
\{a\}^\mI(x) & \!\!=\!\! & \textrm{(if $x \neq a^\mI$ then 0 else 1)}\\
(\triangle C)^\mI(x) & \!\!=\!\! & \triangle (C^\mI(x)) \\
(\lnot C)^\mI(x) & \!\!=\!\! & \fneg (C^\mI(x)) \\
(C \mand D)^\mI(x) & \!\!=\!\! & C^\mI(x) \fand D^\mI(x) \\
(C \mor D)^\mI(x) & \!\!=\!\! & C^\mI(x) \fOr D^\mI(x) \\
(C \to D)^\mI(x) & \!\!=\!\! & (C^\mI(x) \fto D^\mI(x)) \\
(\V r.C)^\mI(x) & \!\!=\!\! & \inf \{r^\mI(x,y) \fto C^\mI(y) \mid y \in \Delta^\mI\} \\
(\E r.C)^\mI(x) & \!\!=\!\! & \sup \{r^\mI(x,y) \fand C^\mI(y) \mid y \in \Delta^\mI\}.
\end{array}
\]
} %\ramka
\caption{The meaning of complex roles and concepts in a fuzzy interpretation~$\mI$.\label{fig: HDHGS}}
\end{figure}

If $x,y \in \Delta^\mI$ and $r^\mI(x,y) > 0$ (respectively, $r^\mI(x,y) = 0$), then we say that $\tuple{x,y}$ is a {\em nonzero instance} (respectively, {\em zero-instance}) of the role~$r$. 
We say that concepts $C$ and $D$ are {\em equivalent}, denoted by $C \equiv D$, if $C^\mI = D^\mI$ for every fuzzy interpretation~$\mI$. 

\begin{example}
Let $\leq$ be the usual order on $L = [0,1]$ and let $\CN = \{A\}$, $\RN = \{r\}$ and $\IN = \emptyset$. 
Consider the fuzzy interpretation $\mI$ illustrated and specified below: 
	
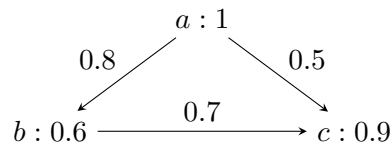
\begin{center}
		\begin{tikzpicture}[->,>=stealth]
			\node (a) {$a: 1$};
			\node (ua) [node distance=1.5cm, below of=a] {};
			\node (b) [node distance=2cm, left of=ua] {$b: 0.6$};
			\node (c) [node distance=2cm, right of=ua] {$c: 0.9$};
			\draw (a) to node [left,yshift=2mm]{0.8} (b);
			\draw (a) to node [right,yshift=2mm]{0.5} (c);
			\draw (b) to node [above]{0.7} (c);
		\end{tikzpicture}
	\end{center}	
	$\Delta^\mI = \{a,b,c\}$, 
	$A^\mI = \{a\!:\!1,\ b\!:\!0.6,\ c\!:\!0.9\}$ and
	$r^\mI = \{\tuple{a,b}\!:\!0.8$, $\tuple{a,c}\!:\!0.5$, $\tuple{b,c}\!:\!0.7\}$. 
	Denote $r^+ = r^* \circ r$. We have: 
	
	\smallskip
	
	\begin{center}
		\begin{tabular}{|c|c|c|c|}
			\hline
			& \ \ \ \ G\"odel\ \ \ \ & \ \ Product\ \ \ & {\L}ukasiewicz \\
			\hline
			& & & \\[-2.2ex]
			$(\E r.A)^\mI(a)$ & $0.6$ & $0.48$ & $0.4$ \\
			\hline
			& & & \\[-2.2ex]
			$(\V r.A)^\mI(a)$ & $0.6$ & $0.75$ & $0.8$ \\
			\hline
			& & & \\[-2.2ex]
			$(\E r^+.A)^\mI(a)$ & $0.7$ & $0.504$ & $0.4$ \\
			\hline
			& & & \\[-2.2ex]
			$(\V r^+.A)^\mI(a)$ & $0.6$ & $0.75$ & $0.8$ \\
			\hline
		\end{tabular}
	\end{center}
	
\end{example}

%-------------------------------------------------------------------------------------------------------------------------------------

A {\em fuzzy assertion} in \fALCP has the form $C(a) \bowtie p$, $R(a,b) \bowtie p$, $a \doteq b$ or $a \not\doteq b$, where $C$ is a concept of \fALCP, $R$ is a role of \fALCP, $\bowtie\ \in \{\geq, >, \leq, <\}$ and $p \in \IZOcc$. A~{\em fuzzy ABox} in \fALCP is a finite set of fuzzy assertions in \fALCP. 

A {\em fuzzy TBox axiom} in \fALCP has the form $(C \sqsubseteq D) \rhd p$, where $C$ and $D$ are concepts of \fALCP, $\rhd \in \{\geq, > \}$ and $p \in \IZOcc$. A {\em fuzzy TBox} in \fALCP is a finite set of fuzzy TBox axioms in~\fALCP. 

The satisfaction relation $\mI \models \varphi$ ($\mI$ {\em validates}~$\varphi$), where $\mI$ is a fuzzy interpretation and $\varphi$ is a TBox axiom or a fuzzy assertion, is defined as follows:
\[
\begin{array}{lcl}
\mI \models (C \sqsubseteq D) \rhd p & \!\!\textrm{iff}\!\! & (C \to D)^\mI(x) \rhd p \textrm{ for all } x \in \Delta^\mI, \\
\mI \models C(a) \bowtie p & \!\!\textrm{iff}\!\! & C^\mI(a^\mI) \bowtie p, \\ 
\mI \models R(a,b) \bowtie p & \!\!\textrm{iff}\!\! & R^\mI(a^\mI,b^\mI) \bowtie p, \\
\mI \models a \doteq b & \!\!\textrm{iff}\!\! & a^\mI = b^\mI, \\ 
\mI \models a \not\doteq b & \!\!\textrm{iff}\!\! & a^\mI \neq b^\mI.
\end{array}
\]

A fuzzy interpretation $\mI$ is a {\em model} of a fuzzy TBox $\mT$, denoted by $\mI \models \mT$, if $\mI \models \varphi$ for all $\varphi \in \mT$. Similarly, $\mI$ is a {\em model} of a fuzzy ABox $\mA$, denoted by $\mI \models \mA$, if $\mI \models \varphi$ for all $\varphi \in \mA$.

\comment{
A role $R$ of \fALCP is in {\em inverse normal form} if the inverse constructor is applied in $R$ only to role names (i.e., atomic roles). 
Given roles $R$ and $S$, we write $R \equiv S$ and say that $R$ and $S$ are {\em equivalent} if $R^\mI = S^\mI$ for every fuzzy interpretation~$\mI$. 

\begin{lemma}\label{lemma: JHFAS}
	Every role of \fALCP has an equivalent in inverse normal form.
\end{lemma}

Without loss of generality, from now on we assume that roles are in inverse normal form. 
} %\comment

%===============================================================================

\subsection{Crisp Bisimulations for Fuzzy Description Logics}
\label{section: crisp bis}

Given fuzzy interpretations $\mI$ and $\mI'$, a relation $Z \subseteq \Delta^\mI \times \Delta^\mIp$ is called a ({\em crisp}) {\em $\Phi$-bisimulation} between $\mI$ and $\mI'$ if the following conditions hold for every $x,y \in \Delta^\mI$, $x',y' \in \Delta^\mIp$, $A \in \CN$, $r \in \RN$, $a \in \IN$, and every basic role $R$ w.r.t.\ $\Phi$, where $\to$, $\land$ and $\leftrightarrow$ are the usual crisp logical connectives:
\begin{eqnarray}
&&\!\!\!\!\!\!\!\!\!\!\!\!\!\!\!\!\!\!\!\! Z(x,x') \to A^\mI(x) = A^\mIp(x'), \label{eq: FB 2} \\[0.5ex]
&&\!\!\!\!\!\!\!\!\!\!\!\!\!\!\!\!\!\!\!\! [Z(x,x') \,\land\, (R^\mI(x,y) \!>\! 0)] \to \E y' \in \Delta^\mIp\, [(R^\mI(x,y) \!\leq\! R^\mIp(x',y')) \,\land\, Z(y,y')], \label{eq: FS 3} \\[0.5ex]
&&\!\!\!\!\!\!\!\!\!\!\!\!\!\!\!\!\!\!\!\! [Z(x,x') \,\land\, (R^\mIp(x',y') \!>\! 0)] \to \E y \in \Delta^\mI\,[(R^\mIp(x',y') \!\leq\! R^\mI(x,y)) \,\land\, Z(y,y')]; \label{eq: FS 3b}
\end{eqnarray}
if $O \in \Phi$, then 
\begin{eqnarray}
&&\!\!\!\!\!\!\!\!\!\!\!\!\!\!\!\!\!\!\!\! Z(x,x') \to (x = a^\mI \leftrightarrow x' = a^\mIp); \label{eq: FS 4bis}
\end{eqnarray}
if $U \in \Phi$, then 
\begin{eqnarray}
&&\!\!\!\!\!\!\!\!\!\!\!\!\!\!\!\!\!\!\!\! Z \neq \emptyset \to \V x \in \Delta^\mI\, \E x' \in \Delta^\mIp\ Z(x,x'), \label{eq: FS 6} \\
&&\!\!\!\!\!\!\!\!\!\!\!\!\!\!\!\!\!\!\!\! Z \neq \emptyset \to \V x' \in \Delta^\mIp\, \E x \in \Delta^\mI\ Z(x,x'). \label{eq: FS 6b}
\end{eqnarray}

Note that the definition of $\Phi$-bisimulations does not depend on whether the elements of $\{\circ$, $\mor_r$, $*$, $?\}$ belong to $\Phi$. It depends on whether $I \in \Phi$ via basic roles w.r.t.~$\Phi$. 

A {\em $\Phi$-auto-bisimulation} of a fuzzy interpretation~$\mI$ is a $\Phi$-bisimulation between $\mI$ and itself.

\begin{figure}
	\begin{center}
		\begin{tikzpicture}[->,>=stealth,auto]
			\node (S) {$\mI$};
			\node (u) [node distance=0.8cm, below of=S] {$u\!:\!1$};
			\node (bu) [node distance=2.0cm, below of=u] {};		
			\node (v) [node distance=1.5cm, left of=bu] {$v\!:\!0.5$};
			\node (w) [node distance=1.5cm, right of=bu] {$w\!:\!0.5$};
			\draw (u) to node[left]{0.7} (v);
			\draw (u) to node[right]{0.9} (w);
			\draw (v) edge [bend left=20] node[above]{0.8} (w);
			\draw (w) edge [bend left=20] node[below]{0.8} (v);
			\draw (v) edge[loop below,out=-120,in=-60,looseness=10] node{$0.6$} (v);
			\node (Sp) [node distance=5cm, right of=S] {$\mIp$};
			\node (up) [node distance=0.8cm, below of=Sp] {$u'\!:\!1$};
			\node (vp) [node distance=2.0cm, below of=up] {$v'\!:\!0.5$};
			\draw (up) to node[right]{0.9} (vp);
			\draw (vp) edge[loop below,out=-120,in=-60,looseness=10] node{$0.8$} (vp);
		\end{tikzpicture}
	\caption{An illustration for Example~\ref{example: HDJQK}.\label{fig: GDHSJ}}
	\end{center}
\end{figure}

\begin{example}\label{example: HDJQK}
Let $\leq$ be the usual order on $L = [0,1]$ and let $\CN = \{A\}$, $\RN = \{r\}$ and $\IN = \{a\}$. 
Consider the fuzzy interpretations $\mI$ and $\mIp$ specified below and illustrated in Figure~\ref{fig: GDHSJ}. 
\begin{itemize}
\item $\Delta^\mI = \{u,v,w\}$, $a^\mI = u$, $A^\mI = \{u\!:\!1, v\!:\!0.5, w\!:\!0.5\}$, \\
$r^\mI = \{\tuple{u,v}\!:\!0.7, \tuple{u,w}\!:\!0.9, \tuple{v,v}\!:\!0.6, \tuple{v,w}\!:\!0.8, \tuple{w,v}\!:\!0.8\}$;
\item $\Delta^\mIp = \{u',v'\}$, $a^\mIp = u'$, $A^\mIp = \{u\!:\!1, v\!:\!0.5\}$, 
$r^\mIp = \{\tuple{u',v'}\!:\!0.9, \tuple{v',v'}\!:\!0.8\}$. 
\end{itemize}
Let $\{\triangle\} \subseteq \Phi \subseteq \{\triangle, \circ, \mor_r, *, ?, U, O\}$. It is straightforward to check that $Z = \{\tuple{u,u'}$, $\tuple{v,v'}$, $\tuple{w,v'}\}$ is the largest $\Phi$-bisimulation between $\mI$ and $\mI'$.
\myend
\end{example}

\begin{proposition}[{\cite[Proposition~5]{TFS2022}}]\label{prop: HFHSJ 2}
Let $\mI$, $\mIp$ and $\mI''$ be fuzzy interpretations.
\begin{enumerate}
\item\label{item: HFHSJ2 1} $\{\tuple{x,x} \mid x \in \Delta^\mI\}$ is a $\Phi$-auto-bisimulation of~$\mI$.
		
\item\label{item: HFHSJ2 2}  If $Z$ is a $\Phi$-bisimulation between $\mI$ and $\mIp$, then $Z^{-1}$ is a $\Phi$-bisimulation between $\mIp$ and~$\mI$.
		
\item\label{item: HFHSJ2 3}  If $Z_1$ is a $\Phi$-bisimulation between $\mI$ and $\mIp$, and $Z_2$ is a $\Phi$-bisimulation between $\mIp$ and $\mI''$, then $Z_1 \circ Z_2$ is a $\Phi$-bisimulation between $\mI$ and $\mI''$.
		
\item\label{item: HFHSJ2 4}  If $\mZ$ is a set of $\Phi$-bisimulations between $\mI$ and $\mIp$, then $\bigcup\mZ$ is also a $\Phi$-bisimulation between $\mI$ and $\mIp$.
\end{enumerate}   
\end{proposition}

\begin{corollary}[{\cite[Corollary~6]{TFS2022}}]\label{prop: HDFER}
For any fuzzy interpretation $\mI$, the largest $\Phi$-auto-bisimulation of $\mI$ exists and is an equivalence relation. 
\end{corollary}

Let $\mI$ and $\mIp$ be fuzzy interpretations and let $x \in \Delta^\mI$ and $x' \in \Delta^\mIp$. We write $x \simP x'$ to denote that there exists a $\Phi$-bisimulation $Z$ between $\mI$ and $\mI'$ such that $Z(x,x')$ holds. 
We also write $x \equivP x'$ to denote that $C^\mI(x) = C^\mIp(x')$ for all concepts $C$ of~\fALCP. 
If $x \simP x'$, then we say that $x$ and $x'$ are {\em $\Phi$-bisimilar}. 
If $\IN \neq \emptyset$ and there exists a $\Phi$-bisimulation $Z$ between $\mI$ and $\mI'$ such that $Z(a^\mI,a^\mIp)$ holds for all $a \in \IN$, then we say that $\mI$ and $\mI'$ are {\em $\Phi$-bisimilar} and write $\mI \simP \mIp$. 

The work~\cite{TFS2022} provides results on invariance of concepts under bisimulations and the Hennessy-Milner property of bisimulations. They are formulated for ``witnessed'' and ``modally saturated'' fuzzy interpretations. For a special case, where the considered fuzzy interpretations are finite, they imply the following theorem.

\begin{theorem}[{cf.~\cite[Corollary~17]{TFS2022}}]\label{theorem: HGDHS}
Let $\mI$ and $\mI'$ be finite fuzzy interpretations and let $x \in \Delta^\mI$ and $x' \in \Delta^\mIp$. 
Then $x \equivP x'$ iff $x \simP x'$.
\end{theorem}

\begin{corollary}[{cf.~\cite[Corollary~18]{TFS2022}}]\label{cor: fG H-M 4 c}
Suppose $\IN \neq \emptyset$ and let $\mI$ and $\mI'$ be finite fuzzy interpretations. Then, $\mI$ and $\mIp$ are $\Phi$-bisimilar iff $a^\mI \equivP a^\mIp$ for all $a \in \IN$.
\end{corollary}

This corollary follows from Theorem~\ref{theorem: HGDHS} and the fourth assertion of Proposition~\ref{prop: HFHSJ 2}.

A fuzzy interpretation $\mI$ is said to be {\em $\Phi$-connected} if, for every $x \in \Delta^\mI$, there exist $a \in \IN$, $x_0,\ldots,x_n \in \Delta^\mI$ and basic roles $R_1,\ldots,R_n$ w.r.t.\ $\Phi$ such that $x_0 = a^\mI$, $x_n = x$ and $R_i^\mI(x_{i-1},x_i) > 0$ for all $1 \leq i \leq n$. 

A fuzzy TBox or ABox is {\em invariant under $\Phi$-bisimilarity between finite fuzzy interpretations} if, for every finite fuzzy interpretations $\mI$ and $\mI'$ that are $\Phi$-bisimilar, $\mI$ is a model of the box iff so is~$\mI'$. 

The three following propositions are special cases of Propositions~11--13 of~\cite{TFS2022}. We present full proofs for them, as the latter are provided in~\cite{TFS2022} without proofs. 

\begin{proposition}\label{prop: HGFKW}
If $U \in \Phi$, then all fuzzy TBoxes in \fALCP are invariant under $\Phi$-bisimilarity between finite fuzzy interpretations.
\end{proposition}

\begin{proof}
Suppose $U \in \Phi$ and let $\mI$ and $\mI'$ be finite fuzzy interpretations that are $\Phi$-bisimilar. It is assumed here that $\IN \neq \emptyset$. Let $\mT$ be a fuzzy TBox in $\fALCP$ and suppose $\mI \models \mT$. Let $(C \sqsubseteq D) \rhd p$ be a fuzzy TBox axiom from $\mT$ and let $x' \in \Delta^\mIp$. To prove that $\mIp \models \mT$, we need to show that $(C \to D)^\mIp(x') \rhd p$. Let $Z$ be a $\Phi$-bisimulation between $\mI$ and $\mIp$ such that $Z(a^\mI,a^\mIp)$ holds for all $a \in \IN$. Thus, $Z \neq \emptyset$. By Condition~\eqref{eq: FS 6b}, there exists $x \in \Delta^\mI$ such that $Z(x,x')$ holds. By Theorem~\ref{theorem: HGDHS}, it follows that $(C \to D)^\mI(x) = (C \to D)^\mIp(x')$. Since $\mI \models \mT$, we have that $(C \to D)^\mI(x) \rhd p$. Hence, $(C \to D)^\mIp(x') \rhd p$. 
\myend
\end{proof}

\begin{proposition}\label{prop: UFSSK 2}
Let $\mT$ be a fuzzy TBox in \fALCP and let $\mI$ and $\mIp$ be finite fuzzy interpretations that are $\Phi$-bisimilar. If both $\mI$ and $\mIp$ are $\Phi$-connected, then $\mI \models \mT$ iff $\mI' \models \mT$.
\end{proposition}

\begin{proof}
Suppose both $\mI$ and $\mIp$ are $\Phi$-connected. We prove that, if $\mI \models \mT$, then $\mI' \models \mT$. The converse is similar and omitted. Suppose $\mI \models \mT$. Let $(C \sqsubseteq D) \rhd p$ be an arbitrary fuzzy TBox axiom of $\mT$ and let $x' \in \Delta^\mIp$. We need to show that $(C \to D)^\mIp(x') \rhd p$. 
Since $\mIp$ is $\Phi$-connected, there exists $a \in \IN$, $x'_0,\ldots,x'_n \in \Delta^\mIp$ and basic roles $R_1,\ldots,R_n$ w.r.t.\ $\Phi$ such that $x'_0 = a^\mIp$, $x'_n = x'$ and $R_i^\mIp(x'_{i-1},x'_i) > 0$ for all $1 \leq i \leq n$. Let $x_0 = a^\mI$ and let $Z$ be a $\Phi$-bisimulation between $\mI$ and $\mI'$ such that $Z(b^\mI,b^\mIp)$ holds for all $b \in \IN$. 
For each $i$ from $1$ to $n$, since $Z(x_{i-1},x'_{i-1})$ holds and $R^\mIp(x'_{i-1},x'_i) > 0$, by Condition~\eqref{eq: FS 3b}, there exists $x_i \in \Delta^\mI$ such that $Z(x_i,x'_i)$ holds. Let $x = x_n$. Thus, $Z(x,x')$ holds. Since $\mI \models \mT$, we have $(C \to D)^\mI(x) \rhd p$. Since $Z(x,x')$ holds, by Theorem~\ref{theorem: HGDHS}, $(C \to D)^\mIp(x') = (C \to D)^\mI(x) \rhd p$, which completes the proof.
\myend
\end{proof}

\begin{proposition}\label{prop: IFDMS}
Let $\mA$ be a fuzzy ABox in \fALCP. If $O \in \Phi$ or $\mA$ consists of only fuzzy assertions of the form $C(a) \bowtie p$, then $\mA$ is invariant under $\Phi$-bisimilarity between finite fuzzy interpretations.
\end{proposition}

\begin{proof}
Suppose that $O \in \Phi$ or $\mA$ consists of only fuzzy assertions of the form $C(a) \bowtie p$. Let $\mI$ and $\mI'$ be finite fuzzy interpretations. Suppose that $\mI \simP \mI'$ and $\mI \models \mA$. Let $\varphi \in \mA$. We need to show that $\mI' \models \varphi$. Let $Z$ be a $\Phi$-bisimulation between $\mI$ and $\mIp$ such that $Z(a^\mI,a^\mIp)$ holds for all $a \in \IN$. 
\begin{itemize}
\item Case $\varphi = (a \doteq b)\,$: Since $\mI \models \mA$, we have $a^\mI = b^\mI$. Since $\mI \simP \mI'$, we have $a^\mI \simP a^{\mI'}$. Since $a^\mI = b^\mI$, by Condition~\eqref{eq: FS 4bis}, it follows that $a^{\mI'} = b^{\mI'}$. Therefore, $\mI' \models \varphi$. 
		
\item Case $\varphi = (a \not\doteq b)$ is similar to the above one.
		
\item Case $\varphi = (C(a) \bowtie  p)\,$: Since $\mI \models \mA$, $C^\mI(a^\mI) \bowtie p$. Since $\mI \simP \mI'$, we have $a^\mI \simP a^{\mI'}$. By Theorem~\ref{theorem: HGDHS}, it follows that $C^{\mI'}(a^{\mI'}) = C^\mI(a^\mI) \bowtie p$. Hence, $\mI' \models \varphi$. 
		
\item Case $\varphi = (R(a,b) \rhd p)$, with $\rhd \in \{\geq,>\}$: Without loss of generality, we assume that either $p > 0$ or $\rhd$ is $>$. Since $\mI \models \mA$, $R^\mI(a^\mI,b^\mI) \rhd p$. By~\cite[Lemma~9]{TFS2022}, there exists $y' \in \Delta^\mIp$ such that $Z(b^\mI,y')$ holds and $R^\mIp(a^\mIp,y') \rhd p$. Consider $C = \{b\}$. Since $Z(b^\mI,y')$ holds, by Theorem~\ref{theorem: HGDHS}, $C^{\mI'}(y') = C^\mI(b^\mI) = 1$. Thus, $y' = b^\mIp$ and $\mI' \models \varphi$. 
		
\item Case $\varphi = (R(a,b) \lhd p)$, with $\lhd \in \{\leq,<\}$: For a contradiction, suppose $\mI' \not\models \varphi$. Thus, $R^\mIp(a^\mIp,b^\mIp) \rhd p$, where $\rhd$ is \mbox{$\not\!\!\lhd$}. Similarly to the above case, we can derive that $\mI \models (R(a,b) \rhd p)$, which contradicts $\mI \models \varphi$.
\myend
\end{itemize}
\end{proof}

%===============================================================================

\section{Minimizing Finite Fuzzy Interpretations: Logical Properties}
\label{section: minimization}

In~\cite[Section~6]{FSS2020} together with coauthors we provided results on how to minimize a fuzzy interpretation w.r.t.\ certain logical criteria expressed in FDLs extended with involutive negation under the G\"odel semantics. The minimization is done by taking the quotient fuzzy interpretation w.r.t.\ the equivalence relation that is the largest crisp $\Phi$-auto-bisimulation. Apart from that definition, the main results of~\cite[Section~6]{FSS2020} provide logical properties of such minimized fuzzy interpretations. In this section, we generalize those results for the \FDLs \fALCP with $\triangle \in \Phi$ over an arbitrary algebra of fuzzy truth values. 
Definition~\ref{def: HDJOS}, Lemma~\ref{lemma: HDAMA}, Corollary~\ref{cor: JFWKA}, as well as the assertion~\ref{item: GHDJW} of Theorem~\ref{theorem: HSJAO} and the assertion~\ref{item: HFJHW 1} of Corollary~\ref{cor: HFJHW} given in this section are direct reformulations of the results of~\cite[Section~6]{FSS2020} and are not contributions of the current work. They depend only on the definition of crisp bisimulations, but not on the considered FDLs. 
On the other hand, Corollary~\ref{cor: JFWKA2}, the assertions~\ref{item: GHDJW0} and~\ref{item: GHDJW2} of Theorem~\ref{theorem: HSJAO} and the assertion~\ref{item: HFJHW 2} of Corollary~\ref{cor: HFJHW} depend on the considered FDLs, which differ from 
the FDLs studied in~\cite{FSS2020} in two essential aspects: the Baaz projection operator is used instead of involutive negation and an arbitrary algebra of fuzzy truth values is used instead of the G\"odel structure. They are our contributions for this section.

Given a finite fuzzy interpretation $\mI$, by $\simPI$ we denote the binary relation on $\Delta^\mI$ such that, for $x,x' \in \Delta^\mI$, $x \simPI x'$ iff $x \simP x'$. We call it the {\em $\Phi$-bisimilarity relation of~$\mI$}. 
By Proposition~\ref{prop: HFHSJ 2}, it is the largest $\Phi$-auto-bisimulation of~$\mI$. 
By Theorem~\ref{theorem: HGDHS}, it is an equivalence relation on $\Delta^\mI$. 

\begin{Definition}%[{cf.~\cite[Definition~6.1]{FSS2020}}]
\label{def: HDJOS}\em
Given a finite fuzzy interpretation $\mI$, the {\em quotient fuzzy interpretation} $\mIsimP$ of $\mI$ w.r.t.\ the equivalence relation $\simPI$ is specified as follows:\footnote{We write $\mIsimP$ instead of~$\mI/_{\simPI}$ for simplicity.}	 
\begin{itemize}
\item $\Delta^{\mIsimP} = \{[x]_{\simPI} \mid x \in \Delta^\mI\}$, where $[x]_{\simPI}$ is the equivalence class of $x$ w.r.t.\ $\simPI$, 
\item $a^{\mIsimP} = [a^\mI]_{\simPI}$ for $a \in \IN$, 
\item $A^{\mIsimP}([x]_{\simPI}) = A^\mI(x)$ for $A \in \CN$ and $x \in \Delta^\mI$, 
\item $r^{\mIsimP}([x]_{\simPI},[y]_{\simPI}) = \sup\{r^\mI(x,y') \mid y' \in [y]_{\simPI}\}$ for $r \in \RN$ and $x,y \in \Delta^\mI$.
\myend
\end{itemize}
\end{Definition}
Observe the following:
\begin{itemize}
\item For every $A \in \CN$, $x \in \Delta^\mI$ and $x' \in [x]_{\simPI}$, $A^\mI(x) = A^\mI(x')$. 
This follows from Condition~\eqref{eq: FB 2}.
\item For every $r \in \RN$, $x,y \in \Delta^\mI$ and $x' \in [x]_{\simPI}$, 
\[ 
	\sup\{r^\mI(x,y') \mid y' \in [y]_{\simPI}\} = 
	\sup\{r^\mI(x',y') \mid y' \in [y]_{\simPI}\}.
\]
This follows from Conditions~\eqref{eq: FS 3} and \eqref{eq: FS 3b}, when $\simPI$ is used as~$Z$.
\end{itemize}
Therefore, Definition~\ref{def: HDJOS} is well specified.

\begin{figure}
	\begin{center}
		\begin{tikzpicture}[->,>=stealth,auto]
			\node (S) {$\mI$};
			\node (u) [node distance=0.8cm, below of=S] {$u\!:\!1$};
			\node (bu) [node distance=2.0cm, below of=u] {};		
			\node (v) [node distance=1.5cm, left of=bu] {$v\!:\!0.5$};
			\node (w) [node distance=1.5cm, right of=bu] {$w\!:\!0.5$};
			\draw (u) to node[left]{0.7} (v);
			\draw (u) to node[right]{0.9} (w);
			\draw (v) edge [bend left=20] node[above]{0.8} (w);
			\draw (w) edge [bend left=20] node[below]{0.8} (v);
			\draw (v) edge[loop below,out=-120,in=-60,looseness=10] node{$0.6$} (v);
			\node (Sp) [node distance=5cm, right of=S] {$\mIsimP$};
			\node (up) [node distance=0.8cm, below of=Sp] {$\{u\}\!:\!1$};
			\node (vp) [node distance=2.0cm, below of=up] {$\{v,w\}\!:\!0.5$};
			\draw (up) to node[right]{0.9} (vp);
			\draw (vp) edge[loop below,out=-120,in=-60,looseness=10] node{$0.8$} (vp);
		\end{tikzpicture}
		\caption{An illustration for Example~\ref{example: HDJQK2}.\label{fig: GDHSJ2}}
	\end{center}
\end{figure}

\begin{example}\label{example: HDJQK2}
Let $\leq$ be the usual order on $L = [0,1]$ and let $\CN = \{A\}$, $\RN = \{r\}$ and $\IN = \{a\}$. 
Reconsider the fuzzy interpretation $\mI$ specified in Example~\ref{example: HDJQK}. It is redisplayed in Figure~\ref{fig: GDHSJ2}. 
Let $\{\triangle\} \subseteq \Phi \subseteq \{\triangle, \circ, \mor_r, *, ?, U, O\}$. It is straightforward to check that $\{\tuple{u,u}$, $\tuple{v,v}$, $\tuple{w,w}$, $\tuple{v,w}$, $\tuple{w,v}\}$ is the largest $\Phi$-auto-bisimulation of $\mI$. Consider $\mIsimP$. We have $\Delta^{\mIsimP} = \{\{u\}, \{v,w\}\}$, $a^{\mIsimP} = \{u\}$, $A^{\mIsimP} = \{\{u\}\!:\!1, \{v,w\}\!:\!0.5\}$ and $r^{\mIsimP} = \{\tuple{\{u\},\{v,w\}}\!:\!0.9, \tuple{\{v,w\},\{v,w\}}\!:\!0.8\}$. The fuzzy interpretation $\mIsimP$ is illustrated in Figure~\ref{fig: GDHSJ2}. It is isomorphic to the fuzzy interpretation $\mIp$ specified in Example~\ref{example: HDJQK} and illustrated in Figure~\ref{fig: GDHSJ}.
\myend
\end{example}

\begin{lemma}%[{cf.~\cite[Lemma~6.1]{FSS2020}}]
\label{lemma: HDAMA}
Let $\mI$ be a finite fuzzy interpretation and \mbox{$Z = \{ \tuple{x,[x]_{\simPI}} \mid x \in \Delta^\mI\}$} be a subset of $\Delta^\mI \times \Delta^{\mIsimP}$. Then, $Z$ is a $\Phi$-bisimulation as well as a $(\Phi \cup \{U\})$-bisimulation between $\mI$ and $\mIsimP$.
\end{lemma}

\begin{proof}
We need to prove Conditions~\eqref{eq: FB 2}--\eqref{eq: FS 3b}, Condition~\eqref{eq: FS 4bis} when $O \in \Phi$, and Conditions~\eqref{eq: FS 6} and~\eqref{eq: FS 6b} regardless of whether $U \in \Phi$, for $\mIp = \mIsimP$. 
	
\begin{itemize}
\item Condition~\eqref{eq: FB 2} directly follows from Definition~\ref{def: HDJOS}. 
		
\item Consider Condition~\eqref{eq: FS 3} and assume that $Z(x,x')$ holds and $R^\mI(x,y) \!>\! 0$. We have $x' = [x]_{\simPI}$. Take $y' = [y]_{\simPI}$. By Definition~\ref{def: HDJOS}, $R^\mI(x,y) \leq R^{\mIsimP}([x]_{\simPI},[y]_{\simPI}) = R^\mIp(x',y')$. Clearly, $Z(y,y')$ holds. 
		
\item Consider Condition~\eqref{eq: FS 3b} and assume that $Z(x,x')$ holds and $R^\mIp(x',y') \!>\! 0$. We have $x' = [x]_{\simPI}$. Since $\mI$ is finite, by Definition~\ref{def: HDJOS}, $R^{\mIsimP}([x]_{\simPI},y') = \max\{R^\mI(x,y) \mid y \in y'\}$. Hence, there exists $y \in y' \subseteq \Delta^\mI$ such that $R^{\mIsimP}([x]_{\simPI},y') = R^\mI(x,y)$, which means $R^\mIp(x',y') = R^\mI(x,y)$. Clearly, $Z(y,y')$ holds.  
		
\item Consider Condition~\eqref{eq: FS 4bis} for the case $O \in \Phi$ and assume that $Z(x,x')$ holds. We have $x' = [x]_{\simPI}$. 
If $x = a^\mI$, then by Definition~\ref{def: HDJOS}, $x' = [x]_{\simPI} = [a^\mI]_{\simPI} = a^\mIp$.  
Conversely, if $x' = a^\mIp$, then $x \simPI a^\mI$ and, by Condition~\eqref{eq: FS 4bis} with $Z$, $x'$ and $\mIp$ replaced by $\simPI$, $a^\mI$ and $\mI$, respectively, we can derive that $x = a^\mI$. 
		
\item Condition~\eqref{eq: FS 6} holds because we can take $x' = [x]_{\simPI}$.
\item Condition~\eqref{eq: FS 6b} holds because we can take any $x \in x'$.
\myend
\end{itemize}
\end{proof}

\begin{corollary}%[{cf.~\cite[Corollary~6.2]{FSS2020}}]
\label{cor: JFWKA}
Let $\mI$ be a finite fuzzy interpretation and suppose \mbox{$\IN \neq \emptyset$}. Then, $\mI$ and $\mIsimP$ are both $\Phi$-bisimilar and $(\Phi \cup \{U\})$-bisimilar.
\end{corollary}

This corollary immediately follows from Lemma~\ref{lemma: HDAMA}. 

\begin{corollary}%[{cf.~\cite[Corollary~6.3]{FSS2020}}]
\label{cor: JFWKA2}
Let $\mI$ be a finite fuzzy interpretation, $\mT$ a fuzzy TBox and $\mA$ a fuzzy ABox in \fALCP. Then:
\begin{enumerate}
\item $\mI \models \mT$ iff $\mIsimP \models \mT$, 
\item if $O \in \Phi$ or $\mA$ consists of only fuzzy assertions of the form $C(a) \bowtie p$, then $\mI \models \mA$ iff $\mIsimP \models \mA$.  
\end{enumerate}
\end{corollary}

\begin{proof}
Without loss of generality, assume that $\IN \neq \emptyset$. 
By Corollary~\ref{cor: JFWKA}, $\mI$ and $\mIsimP$ are both $\Phi$-bisimilar and $(\Phi \cup \{U\})$-bisimilar.

By Proposition~\ref{prop: HGFKW}, all fuzzy TBoxes in \fALCP are invariant under $(\Phi \cup \{U\})$-bisimilarity between finite fuzzy interpretations. 
%In particular, $\mT$ is invariant under $(\Phi \cup \{U\})$-bisimilarity between finite fuzzy interpretations. 
Hence, $\mI \models \mT$ iff $\mIsimP \models \mT$. 

For the second assertion, suppose that $O \in \Phi$ or $\mA$ consists of only fuzzy assertions of the form $C(a) \bowtie p$. By Proposition~\ref{prop: IFDMS}, $\mA$ is invariant under $\Phi$-bisimilarity between finite fuzzy interpretations. Hence, $\mI \models \mA$ iff $\mIsimP \models \mA$.
\myend
\end{proof}

The minimality in the following theorem is defined using the order in which fuzzy interpretations are compared w.r.t.\ the sizes of their domains.

\begin{theorem}%[{cf.~\cite[Theorem 6.4]{FSS2020}}]
\label{theorem: HSJAO}
Let $\mI$ be a finite fuzzy interpretation. 
\begin{enumerate}
\item\label{item: GHDJW0} If $(1 \fto x) = x$ and ($y \fOr z = 0$ iff $y = z = 0$), for all $x,y,z \in L$, then $\mIsimP$ is a minimal fuzzy interpretation that validates the same set of fuzzy TBox axioms in \fALCP as $\mI$. 
\item If $\IN \neq \emptyset$ and either $U \in \Phi$ or $\mI$ is connected w.r.t.\ $\Phi$, then:
	\begin{enumerate}
	\item\label{item: GHDJW} $\mIsimP$ is a minimal fuzzy interpretation $\Phi$-bisimilar to $\mI$, 
	\item\label{item: GHDJW2} $\mIsimP$ is a minimal fuzzy interpretation that validates the same set of fuzzy assertions of the form $C(a) \bowtie p$ in \fALCP as $\mI$.
	\end{enumerate}
\end{enumerate}
\end{theorem}

The premises of the first assertion hold, e.g., when $L = [0,1]$, $\fand$ is a t-norm, $\fto$ is its residuum and $\fOr$ is the corresponding s-norm.

\begin{proof}
Without loss of generality, assume that $\IN \neq \emptyset$. 
By Corollaries~\ref{cor: JFWKA}	and~\ref{cor: JFWKA2}, $\mIsimP$ validates the same set of fuzzy TBox axioms in \fALCP as $\mI$, is $\Phi$-bisimilar to $\mI$, and validates the same set of fuzzy assertions of the form $C(a) \bowtie p$ in \fALCP as $\mI$. We only need to justify its minimality. 
	
Since $\mI$ is finite, $\mIsimP$ is also finite. 
Let $\Delta^{\mIsimP} = \{v_1,\ldots,v_n\}$, where $v_1,\ldots,v_n$ are pairwise distinct and $v_i = [x_i]_{\simPI}$ with $x_i \in \Delta^\mI$, for $1 \leq i \leq n$. By Lemma~\ref{lemma: HDAMA}, $x_i \simP v_i$ for all $1 \leq i \leq n$. Let $i$ and $j$ be arbitrary indices such that $1 \leq i,j\leq n$ and $i \neq j$. We have $x_i \not\simP x_j$. Since $x_i \simP v_i$ and $x_j \simP v_j$, by Theorem~\ref{theorem: HGDHS}, it follows that $v_i \not\equivP v_j$. Therefore, there exists a concept $C_{i,j}$ of \fALCP such that $C_{i,j}^{\mIsimP}(v_i) \neq C_{i,j}^{\mIsimP}(v_j)$. Let $D_{i,j} = \triangle(C_{i,j} \to C_{i,j}^{\mIsimP}(v_i))$ if $C_{i,j}^{\mIsimP}(v_i) < C_{i,j}^{\mIsimP}(v_j)$, and $D_{i,j} = \triangle(C_{i,j}^{\mIsimP}(v_i) \to C_{i,j})$ otherwise (i.e., when $C_{i,j}^{\mIsimP}(v_i) > C_{i,j}^{\mIsimP}(v_j)$). By~\eqref{eq: fto-when-1} and \eqref{eq: triangle}, we have $D_{i,j}^{\mIsimP}(v_i) = 1$ and $D_{i,j}^{\mIsimP}(v_j) = 0$. Let $D_i = D_{i,1} \mand\ldots\mand D_{i,i-1} \mand D_{i,i+1} \mand\ldots\mand D_{i,n}$. By Proposition~\ref{prop: JHFJH}, $D_i^{\mIsimP}(v_i) = 1$ and $D_i^{\mIsimP}(v_j) = 0$. Let $E = D_1 \mor\ldots\mor D_n$ and $E_i = D_1 \mor\ldots\mor D_{i-1} \mor D_{i+1} \mor\ldots\mor D_n$. 

Consider the first assertion and assume that $(1 \fto x) = x$ and $y \fOr z = 0$ iff $y = z = 0$, for all $x,y,z \in L$. The fuzzy interpretation $\mIsimP$ validates the fuzzy TBox axiom $(1 \sqsubseteq E) > 0$ but does not validate $(1 \sqsubseteq E_i) > 0$ for any $1 \leq i \leq n$. Any other fuzzy interpretation with such properties must have at least $n$ elements in the domain. That is, $\mIsimP$ is a minimal fuzzy interpretation that validates the same set of fuzzy TBox axioms in \fALCP as $\mI$.
	
Consider the second assertion of the theorem and suppose that either $U \in \Phi$ or $\mI$ is connected w.r.t.\ $\Phi$. Let \mbox{$Z = \{ \tuple{x,[x]_{\simPI}} \mid x \in \Delta^\mI\}$}. By Lemma~\ref{lemma: HDAMA}, $Z$ is a $\Phi$-bisimulation between $\mI$ and~$\mIsimP$. 
	
\begin{itemize}
\item Consider the assertion~\ref{item: GHDJW} of the theorem and let $\mIp$ be a finite fuzzy interpretation $\Phi$-bisimilar to~$\mI$. There exists a $\Phi$-bisimulation $Z'$ between $\mIp$ and $\mI$ such that $Z'(a^\mIp, a^\mI)$ holds for all $a \in \IN$. Let $Z'' = Z' \circ Z$. By Proposition~\ref{prop: HFHSJ 2}, $Z''$ is a $\Phi$-bisimulation between $\mIp$ and $\mIsimP$. Furthermore, 
\begin{equation}
Z''(a^\mIp, a^{\mIsimP}) \textrm{ holds for all $a \in \IN$.} \label{eq: HFOWA}
\end{equation}
		
	\begin{itemize}
	\item Consider the case where $\mI$ is connected w.r.t.\ $\Phi$. Thus, $\mIsimP$ is also connected w.r.t.\ $\Phi$. By Condition~\eqref{eq: FS 3b} with $\mI$, $\mIp$ and $Z$ replaced by $\mIp$, $\mIsimP$ and $Z''$, respectively, it follows from \eqref{eq: HFOWA} that, for every $1 \leq i \leq n$, there exists $u_i \in \Delta^\mIp$ such that $Z''(u_i,v_i)$ holds.
	\item Consider the case where $U \in \Phi$. Since $\IN \neq \emptyset$, by \eqref{eq: HFOWA} and Condition~\eqref{eq: FS 6b} with $\mI$, $\mIp$ and $Z$ replaced by $\mIp$, $\mIsimP$ and $Z''$, respectively, for every $1 \leq i \leq n$, there exists $u_i \in \Delta^\mIp$ such that $Z''(u_i,v_i)$ holds.
	\end{itemize}
		
Thus, $u_i \simP v_i$ for all $1 \leq i \leq n$. By Theorem~\ref{theorem: HGDHS}, it follows that $u_i \equivP v_i$ for all $1 \leq i \leq n$. Since $v_i \not\equivP v_j$ for any $i \neq j$, it follows that $u_i \not\equivP u_j$ for any $i \neq j$. Therefore the cardinality of $\Delta^\mIp$ is greater than or equal to $n$. 
		
\item Consider the assertion~\ref{item: GHDJW2} of the theorem and let $\mIp$ be a finite fuzzy interpretation that validates the same set of fuzzy assertions of the form $C(a) \bowtie p$ in \fALCP as $\mI$. Thus, $a^\mIp \equivP a^\mI$ for all $a \in \IN$. By Corollary~\ref{cor: fG H-M 4 c}, $\mIp$ and $\mI$ are $\Phi$-bisimilar. By the assertion~\ref{item: GHDJW} proved above, it follows that the cardinality of $\Delta^\mIp$ is greater than or equal to $n$. 
\myend
\end{itemize}
\end{proof}

Given a fuzzy interpretation $\mI$, we say that an individual $x \in \Delta^\mI$ is {\em $\Phi$-reachable} (from a named individual) if there exist $a \in \IN$, $x_0,\ldots,x_n \in \Delta^\mI$ and basic roles $R_1,\ldots,R_n$ w.r.t.~$\Phi$ such that $x_0 = a^\mI$, $x_n = x$ and $R_i^\mI(x_{i-1},x_i) > 0$ for all $1 \leq i \leq n$. 

\begin{corollary}%[{cf.~\cite[Corollary 6.5]{FSS2020}}]
\label{cor: HFJHW}
Suppose $U \notin \Phi$ and $\IN \neq \emptyset$. Let $\mI$ be a finite fuzzy interpretation and $\mIp$ the fuzzy interpretation obtained from $\mI$ by deleting from the domain all $\Phi$-unreachable individuals and restricting the interpretation function accordingly. Then:
\begin{enumerate}
\item $\mIpsimP$\label{item: HFJHW 1} is a minimal fuzzy interpretation $\Phi$-bisimilar to $\mI$, 
\item $\mIpsimP$\label{item: HFJHW 2} is a minimal fuzzy interpretation that validates the same set of fuzzy assertions of the form $C(a) \bowtie p$ in \fALCP as $\mI$.
\end{enumerate}
\end{corollary}

This corollary follows from Theorem~\ref{theorem: HSJAO}, Propositions~\ref{prop: HFHSJ 2} and~\ref{prop: IFDMS} and the observations that $\mIp$ is connected w.r.t.~$\Phi$ and $\Phi$-bisimilar to $\mI$ by using the $\Phi$-bisimulation $\{\tuple{x,x} \mid x \in \Delta^\mIp\}$.

%===============================================================================

\section{An Efficient Algorithm for Minimizing Finite Fuzzy Interpretations}
\label{section: computation}

In the preprint~\cite{abs-2010-15671}, together with a coauthor we gave an efficient algorithm for computing the partition that corresponds to the largest crisp bisimulation of a given finite fuzzy labeled graph. Its complexity is of order $O((m \log{l} + n) \log{n})$, where $n$, $m$ and $l$ are the number of vertices, the number of nonzero edges and the number of different fuzzy degrees
of edges of the input graph, respectively. 
In this section, we first recall related notions defined in~\cite{abs-2010-15671} and that algorithm, and then exploit it to create a new algorithm with the same complexity order that, given a finite fuzzy interpretation $\mI$ and $\{\triangle\} \subseteq \Phi \subseteq \{\triangle$, $\circ$, $\mor_r$, $*$, $?$, $I$, $U$, $O\}$, computes the quotient fuzzy interpretation $\mIsimP$, which has been proved to be minimal and equivalent to $\mI$ w.r.t.\ certain criteria. 

A {\em fuzzy labeled graph}, called a {\em fuzzy graph} for short, is a structure $G = \tuple{V, E, \Label, \SV, \SE}$, where $V$ is a set of vertices, $\SV$ (respectively, $\SE$) is a set of vertex labels (respectively, edge labels), \mbox{$E: V \times \SE \times V \to L$} is a fuzzy set of labeled edges, and $\Label: V \to (\SV \to L)$ is a function that labels vertices. 
It is {\em finite} if all the sets $V$, $\SV$ and $\SE$ are finite. 
If $E(x,r,y) > 0$ (respectively, $E(x,r,y) = 0$), then $\tuple{x,y}$ is called a {\em nonzero $r$-edge} (respectively, {\em zero-$r$-edge}) of~$G$. 
For $x \in V$, $r \in \SE$ and $Y \subseteq V$, we denote $E(x,r,Y) = \{E(x,r,y) \mid y \in Y\}$. 

In this section, let $G = \tuple{V, E, \Label, \SV, \SE}$ be a finite fuzzy graph. 
We will use 
$\bbP$, $\bbQ$ and $\bbS$ to denote partitions of $V$, 
$X$ and $Y$ to denote non-empty subsets of $V$, 
and $r$ to denote an edge label from $\SE$. 
By $\bbP_0$ we denote the partition of $V$ that corresponds to the equivalence relation 
\[ \textrm{$\{\tuple{x,x'} \in V^2 \mid$ $\Label(x) = \Label(x')$ and $\sup E(x,r,V) = \sup E(x',r,V)$ for all $r \in \SE\}$.} \] 

We say that $\bbP$ is a {\em refinement} of $\bbQ$ if, for every $X \in \bbP$, there exists $Y \in \bbQ$ such that $X \subseteq Y$. In that case we also say that $\bbQ$ is {\em coarser} than $\bbP$. By this definition, every partition is a refinement of itself. 
Given a refinement $\bbP$ of a partition $\bbQ$, a block $Y \in \bbQ$ is {\em compound} w.r.t.\ $\bbP$ if there exists $X \in \bbP$ such that $X \subset Y$. 

A non-empty binary relation $Z \subseteq V \times V$ is called a {\em crisp auto-bisimulation} of $G$, or a {\em bisimulation} of $G$ for short, if the following conditions hold (for all possible values of the free variables): 
\begin{eqnarray}
Z(x,x') & \to & \Label(x) = \Label(x') \label{eq: CB 1} \\[1ex]
[Z(x,x') \land (E(x,r,y) \!>\! 0)] & \to & \E y'\,[(E(x,r,y) \!\leq\! E(x',r,y')) \land Z(y,y')] \label{eq: CB 2} \\[1ex]
[Z(x,x') \land (E(x',r,y') \!>\! 0)] & \to & \E y\,[(E(x',r,y') \!\leq\! E(x,r,y)) \land Z(y,y')], \label{eq: CB 3}
\end{eqnarray}
where $\to$ and $\land$ denote the usual crisp logical connectives. 

The largest bisimulation of a fuzzy graph exists and is an equivalence relation.
By the {\em partition that corresponds to the largest bisimulation of $G$} we mean the partition of $V$ that corresponds to the equivalence relation being the largest bisimulation of~$G$. 

We say that $X$ is {\em stable w.r.t.\ $\tuple{Y,r}$} (and $G$) if $\sup E(x,r,Y) = \sup E(x',r,Y)$ for all $x,x' \in X$. 
A partition $\bbP$ is {\em stable w.r.t.\ $\tuple{Y,r}$} (and $G$) if every $X \in \bbP$ is stable w.r.t.\ $\tuple{Y,r}$. Next, $\bbP$ is {\em stable} (w.r.t.\ $G$) if it is stable w.r.t.\ $\tuple{Y,r}$ for all $Y \in \bbP$ and $r \in \SE$. 

For $\emptyset \subset Y' \subset Y$, by $\itSplit(X,\tuple{Y',Y,r})$ we denote the coarsest partition of $X$ such that each of its blocks is stable w.r.t.\ both $\tuple{Y',r}$ and $\tuple{Y\setminus Y',r}$. 
%Clearly, that partition exists and is computable.  How to implement the function is left for later. If $\bbX = \itSplit(X,\tuple{Y',Y,r})$ contains more than one block, then we say that $X$ is split into $\bbX$ by $\tuple{Y',Y,r}$ (or by $\tuple{Y',r}$ with respect to $Y$ as the context). 
%
We also define 
\[ \itSplit(\bbP,\tuple{Y',Y,r}) = \bigcup \{\itSplit(X,\tuple{Y',Y,r}) \mid X \in \bbP\}. \]
Clearly, this set is the coarsest refinement of $\bbP$ that is stable w.r.t.\ both $\tuple{Y',r}$ and $\tuple{Y\setminus Y',r}$. 

\begin{algorithm}[t]
	\caption{\CompCB\label{algCompCB}}
	\Input{a finite fuzzy graph $G = \tuple{V, E, \Label, \SV, \SE}$.}
	\Output{the partition that corresponds to the largest bisimulation of~$G$.}
	
	\smallskip
	
	let $\bbP = \bbP_0$ and $\bbQ_r = \{V\}$ for all $r \in \SE$\label{step: UIIJW 1}\;
	\lIf{$\bbP = \{V\}$}{\Return $\bbP$} 
	\While{there exists $r \in \SE$ such that $\bbQ_r \neq \bbP$\label{step: UIIJW 3}}{ 
		choose such an $r$ and a compound block $Y \in \bbQ_r$ with respect to $\bbP$\;
		choose a block $Y' \in \bbP$ such that $Y' \subset Y$ and $|Y'| \leq |Y|/2$\label{step: UIIJW 5}\;
		$\bbP := \itSplit(\bbP, \tuple{Y',Y,r})$\label{step: UIIJW 6}\;
		refine $\bbQ_r$ by replacing $Y$ with $Y'$ and $Y \setminus Y'$\label{step: UIIJW 7}\;
	}
\Return $\bbP$\label{step: UIIJW 8}\;
\end{algorithm}

Algorithm~\ref{algCompCB} (on page \pageref{algCompCB}) was designed by us and a coauthor~\cite{abs-2010-15671} for computing the partition that corresponds to the largest bisimulation of~$G$. It exploits the fact that $\bbP$ is the partition that corresponds to the largest bisimulation of $G$ iff it is the coarsest stable refinement of~$\bbP_0$~\cite[Lemma~3.2]{abs-2010-15671}. The algorithm is explained in~\cite{abs-2010-15671} as follows. It starts with initializing $\bbP$ to $\bbP_0$. If $\bbP$ is a singleton, then $\bbP$ is stable and the algorithm returns it as the result. Otherwise, the algorithm repeatedly refines $\bbP$ to make it stable as follows. The algorithm maintains a partition $\bbQ_r$ of $V$, for each $r \in \SE$, such that $\bbP$ is a refinement of $\bbQ_r$ and $\bbP$ is stable with respect to $\tuple{Y,r}$ for all $Y \in \bbQ_r$. 
If at some stage $\bbQ_r = \bbP$ for all $r \in \SE$, then $\bbP$ is stable and the algorithm terminates with that~$\bbP$. The variables $\bbQ_r$ are initialized to $\{V\}$ for all $r \in \SE$. In each iteration of the main loop, the algorithm chooses $\bbQ_r \neq \bbP$, $Y \in \bbQ_r$ and $Y' \in \bbP$ such that $Y' \subset Y$ and $|Y'| \leq |Y|/2$, then it replaces $\bbP$ with \mbox{$\itSplit(\bbP, \tuple{Y',Y,r})$} and replaces $Y$ in $\bbQ_r$ with $Y'$ and $Y \setminus Y'$. In this way, the chosen $\bbQ_r$ is refined (and $\bbP$ may also be  refined), so the loop will terminate after a number of iterations. 
The condition $|Y'| \leq |Y|/2$ reflects the idea ``process the smaller half (or a smaller component) first'' from Hopcroft's algorithm~\cite{Hopcroft71} and Paige and Tarjan's algorithm~\cite{PaigeT87}. Without using it the algorithm still terminates with a correct result, but the condition is essential for reducing the complexity order. 

\begin{example}\label{example: HDJQK3}
Let $G = \tuple{V, E, \Label, \SV, \SE}$ be the fuzzy graph identical to the fuzzy interpretations $\mI$ specified in Example~\ref{example: HDJQK3}, with $\SV = \{A\}$, $\SE = \{r\}$, $V = \{u,v,w\}$, $\Label(u)(A) = 1$, $\Label(v)(A) = \Label(w)(A) = 0.5$ and $E = \{\tuple{u,r,v}\!:\!0.7$, $\tuple{u,r,w}\!:\!0.9$, $\tuple{v,r,v}\!:\!0.6$, $\tuple{v,r,w}\!:\!0.8$, $\tuple{w,r,v}\!:\!0.8\}$. 
To increase the readability, we recall the illustration below.
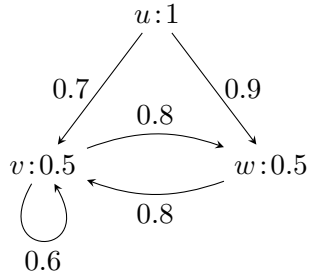
\begin{center}
\begin{tikzpicture}[->,>=stealth,auto]
\node (S) {};
\node (u) [node distance=0.0cm, below of=S] {$u\!:\!1$};
\node (bu) [node distance=2.0cm, below of=u] {};		
\node (v) [node distance=1.5cm, left of=bu] {$v\!:\!0.5$};
\node (w) [node distance=1.5cm, right of=bu] {$w\!:\!0.5$};
\draw (u) to node[left]{0.7} (v);
\draw (u) to node[right]{0.9} (w);
\draw (v) edge [bend left=20] node[above]{0.8} (w);
\draw (w) edge [bend left=20] node[below]{0.8} (v);
\draw (v) edge[loop below,out=-120,in=-60,looseness=10] node{$0.6$} (v);
\end{tikzpicture}
\end{center}
Let's consider the run of Algorithm~\ref{algCompCB} for $G$. After executing the statement~\ref{step: UIIJW 1}, we have $\bbP = \{\{u\},\{v,w\}\}$ and $\bbQ_r = \{\{u,v,w\}\}$. Consider the first iteration of the ``while'' loop. After executing the statement~\ref{step: UIIJW 5} we have $Y = \{u,v,w\}$ and $Y' = \{u\}$. The statement~\ref{step: UIIJW 6} does not make any change to $\bbP$, but the statement~\ref{step: UIIJW 7} makes $\bbQ_r$ equal to $\bbP$. Thus, the loop terminates after the first iteration and the partition returned by the algorithm is $\{\{u\},\{v,w\}\}$.
\myend
\end{example}

It is proved in~\cite{abs-2010-15671} that Algorithm~\ref{algCompCB} is correct. The work~\cite{abs-2010-15671} also specifies how to implement Algorithm~\ref{algCompCB} so that its complexity is of the mentioned order $O((m \log{l} + n) \log{n})$. The implementation uses data structures that allow us to get $\sup E(x,r,Y)$ in time $O(\log{l})$ for every $x, y \in V$ and $r \in \SE$ with $E(x,r,y) > 0$, where $Y$ is the block of the resulting partition $\bbP$ with $y \in Y$ and $l$~is the number of different fuzzy degrees of edges of the input graph.\footnote{If $e$ is the edge $\tuple{x,r,y}$ with $E(x,r,y) > 0$, then $\sup E(x,r,Y) = e.\blockEdge.\maxKey()$.} The data structures also allow us to get in constant time the block of the resulting partition that contains a given vertex $x \in V$, as well as an element of a given block of the resulting partition. 

%===============================================================================

\begin{algorithm}[t]
\caption{\MinInt\label{alg: minInt}}
\Input{a finite fuzzy interpretation $\mI$ and $\{\triangle\} \subseteq \Phi \subseteq \{\triangle, \circ, \mor_r, *, ?, I, U, O\}$.}
\Output{the quotient fuzzy interpretation $\mIsimP$.}
	
\smallskip

construct a fuzzy graph $G = \tuple{V, E, \Label, \SV, \SE}$ such that:\label{step: JHDJH 1}
\begin{tabular}{l}
if $O \notin \Phi$, then $\SV = \CN$, else $\SV = \CN \cup \IN$;\\[0.5ex]
if $I \notin \Phi$, then $\SE = \RN$, else $\SE = \RN \cup \{\cnv{r} \mid r \in \RN\}$;\\[0.5ex]
$V = \Delta^\mI$;\\[0.5ex] 
$E(x,R,y) = R^\mI(x,y)$, for $x,y \in V$ and $R \in \SE$;\\[0.5ex] 
$\Label(x)(A) = A^\mI(x)$, for $x \in V$ and $A \in \CN$;\\[0.5ex] 
$\Label(x)(a)$ = (if $x = a^\mI$ then 1 else 0), for $x \in V$ and $a \in \IN$, when $O \in \Phi$;
\end{tabular}

execute Algorithm~\ref{algCompCB} for $G$ and let $\bbP$ be the obtained partition\label{step: minInt 2};\label{step: JHDJH 2}

let $\mJ$ be the fuzzy interpretation specified as follows:\label{step: JHDJH 3}
\begin{tabular}{l}
$\Delta^\mJ = \bbP$;\\[0.5ex]
for $a \in \IN$, $a^\mJ$ is the block of $\bbP$ that contains $a^\mI$;\\[0.5ex]
for $A \in \CN$ and $X \in \bbP$, $A^\mJ(X) = A^\mI(x)$ for an arbitrary $x \in X$;\\[0.5ex]
for $r \in \RN$ and $x,y \in \Delta^\mI$ with $r^\mI(x,y) > 0$, $r^\mJ(X,Y) = \sup E(x,r,Y)$,\\ 
\mbox{\hspace{1em}}where $X$ and $Y$ are the blocks of $\bbP$ that contain $x$ and $y$, respectively;\\[0.5ex]
for $r \in \RN$ and $X,Y \in \bbP$, if $r^\mJ(X,Y)$ is not set above, then $r^\mJ(X,Y) = 0$;
\end{tabular}

\Return $\mJ$\label{step: JHDJH 4}\;

\end{algorithm}

Algorithm~\ref{alg: minInt} (on page~\pageref{alg: minInt}) is our algorithm for computing the quotient fuzzy interpretation $\mIsimP$ for a given finite fuzzy interpretation $\mI$ and $\{\triangle\} \subseteq \Phi \subseteq \{\triangle, \circ, \mor_r, *, ?, I, U, O\}$. 
It first constructs a fuzzy graph $G = \tuple{V, E, \Label, \SV, \SE}$ that corresponds to $\mI$ w.r.t.~$\Phi$. In particular, $V = \Delta^\mI$, each concept name $A$ is used as a vertex label, with $\Label(x)(A) = A^\mI(x)$ for $x \in V$, and if $O \in \Phi$, then each individual name $a$ is also used as a vertex label, with $\Label(x)(a)$ set to the truth value of $x = a^\mI$ for $x \in V$. Similarly, each role name $r$ is used as an edge label, with $E(x,r,y) = r^\mI(x,y)$ for $x, y \in V$, and if $I \in \Phi$, then the inverse $\cnv{r}$ of each role name $r$ is also used as an edge label, with $E(x,\cnv{r},y) = r^\mI(y,x)$ for $x, y \in V$. 
Next, Algorithm~\ref{algCompCB} is executed to compute the partition $\bbP$ that corresponds to the largest bisimulation of~$G$. Then, the fuzzy interpretation $\mJ$ intended for the result $\mIsimP$ is specified as follows: each block of $\bbP$ forms (identifies) an element of $\Delta^\mJ$; for each individual name $a$, $a^\mJ$ is the block of $\bbP$ that contains $a^\mI$; for each concept name $A$ and each $X \in \Delta^\mJ$, $A^\mJ(X) = A^\mI(x)$ for an arbitrary $x \in X$; for each role name $r$ and each $X,Y \in \Delta^\mJ$, $r^\mJ(X,Y) = \sup E(x,r,Y)$ for an arbitrary $x \in X$. Since $\bbP$ is stable, the mentioned values $A^\mI(x)$ and $\sup E(x,r,Y)$ are independent from the choice of $x \in X$, and the first element of the block $X$ can be taken as~$x$. For the representation of $\mJ$, we keep only nonzero instances of roles. This means that, for $X, Y \in \Delta^\mJ$, the information about $r^\mJ(X,Y)$ (in the form of a pair $\tuple{Y,r^\mJ(X,Y)}$ in a list involved with $X$, called the adjacency list of $X$) is kept only if $r^\mJ(X,Y) > 0$. 

\begin{figure}
\begin{center}
\begin{tabular}{|l|l|}
\hline
\begin{tikzpicture}[->,>=stealth,auto]
\node (a) {$a\,(o)$};
\node (I) [node distance=2.0cm, right of=a] {$\mI$};
\node (b) [node distance=2.0cm, below of=a] {$b$};
\node (c) [node distance=2.0cm, right of=b] {$c$};
\node (d) [node distance=2.0cm, below of=b] {$d$};
\node (e) [node distance=2.0cm, right of=d] {$e$};
\draw (a) to node[right]{0.8} (b);
\draw (b) to node[above]{0.7} (c);
\draw (b) to node[left]{1} (d);
\draw (c) to node[right]{1} (e);
\draw (d) edge [bend left=20] node[above]{1} (e);
\draw (e) edge [bend left=20] node[below]{1} (d);
\node (a2) [node distance=4.0cm, right of=a] {$a_2$};
\node (b2) [node distance=3.0cm, right of=b] {$b_2$};
\node (b3) [node distance=2.0cm, right of=b2] {$b_3$};
\draw (a2) to node[left]{0.8} (b2);
\draw (a2) to node[right]{0.8} (b3);
\draw (b2) edge[loop below,out=-120,in=-60,looseness=10] node{$1$} (b2);
\draw (b3) edge[loop below,out=-120,in=-60,looseness=10] node{$1$} (b3);
\end{tikzpicture}
&
\begin{tikzpicture}[->,>=stealth,auto]
\node (a) {$\{a\}\,(o)$};
\node (J3) [node distance=2.0cm, right of=a] {$\mJ_3$};
\node (b) [node distance=2.0cm, below of=a] {$\{b\}$};
\node (c) [node distance=2.0cm, right of=b] {$\{c\}$};
\node (d) [node distance=2.0cm, below of=b] {$\{d\}$};
\node (e) [node distance=2.0cm, right of=d] {$\{e\}$};
\draw (a) to node[right]{0.8} (b);
\draw (b) to node[above]{0.7} (c);
\draw (b) to node[left]{1} (d);
\draw (c) to node[right]{1} (e);
\draw (d) edge [bend left=20] node[above]{1} (e);
\draw (e) edge [bend left=20] node[below]{1} (d);
\node (a2) [node distance=4.0cm, right of=a] {$\{a_2\}$};
\node (b23) [node distance=2.0cm, below of=a2] {$\{b_2,b_3\}$};
\draw (a2) to node[right]{0.8} (b23);
\draw (b23) edge[loop below,out=-120,in=-60,looseness=10] node{$1$} (b23);
\end{tikzpicture}
\\
\hline
\begin{tikzpicture}[->,>=stealth,auto]
\node (J1) {$\mJ_1$};
\node (a) [node distance=2.8cm, right of=J1] {$\{a, a_2\}\,(o)$};
\node (b) [node distance=2.0cm, below of=a] {$\{b,b_2,b_3,c,d,e\}$};
\draw (a) to node[right]{0.8} (b);
\draw (b) edge[loop below,out=-120,in=-60,looseness=10] node{$1$} (b);
\end{tikzpicture}
&
\begin{tikzpicture}[->,>=stealth,auto]
\node (J2) {$\mJ_2$};
\node (a) [node distance=1.5cm, right of=J2] {$\{a\}\,(o)$};
\node (btwa) [node distance=1.0cm, right of=a] {};
\node (a2) [node distance=1.0cm, right of=btwa] {$\{a_2\}$};
\node (b) [node distance=2.0cm, below of=btwa] {$\{b,b_2,b_3,c,d,e\}$};
\draw (a) to node[left]{0.8} (b);
\draw (a2) to node[left]{0.8} (b);
\draw (b) edge[loop below,out=-120,in=-60,looseness=10] node{$1$} (b);
\end{tikzpicture}
\\
\hline
\end{tabular}
\caption{An illustration for Example~\ref{example: IUFDJ}.\label{fig: GDWIS}}
\end{center}
\end{figure}
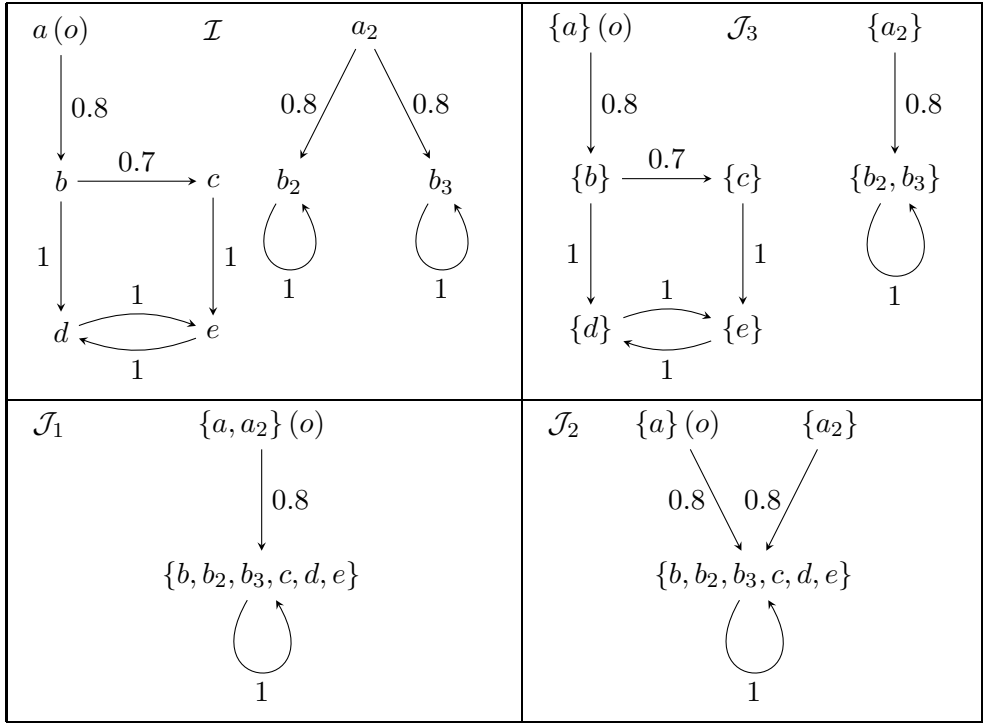

\begin{example}\label{example: IUFDJ}
Let $\leq$ be the usual order on $L = [0,1]$ and let $\CN = \emptyset$, $\RN = \{r\}$ and $\IN = \{o\}$. 
Consider the fuzzy interpretation $\mI$ that is illustrated in the top left corner of Figure~\ref{fig: GDWIS} and specified by: \mbox{$\Delta^\mI = \{a,b,c,d,e,a_2,b_2,b_3\}$}, $o^\mI = a$, $r^\mI = \{\tuple{a,b}\!:\!0.8$, $\tuple{a_2,b_2}\!:\!0.8$, $\tuple{a_2,b_3}\!:\!0.8$, $\tuple{b,c}\!:\!0.7$, $\tuple{b,d}\!:\!1$, $\tuple{c,e}\!:\!1$, $\tuple{d,e}\!:\!1$, $\tuple{e,d}\!:\!1$, $\tuple{b_2,b_2}\!:\!1$, $\tuple{b_3,b_3}\!:\!1\}$. Consider the following cases of $\Phi$ in comparison with $\Psi = \{\triangle, \circ, \mor_r, *, ?, U\}$.
\begin{itemize}
\item Case $\Phi = \Psi$: Consider the execution of Algorithm~\ref{alg: minInt} for $\mI$ and $\Phi$. Let $G = \tuple{V$, $E$, $\Label$, $\SV$, $\SE}$ be the fuzzy graph constructed by the statement~\ref{step: JHDJH 1}. We have $\SV = \emptyset$, $\SE = \{r\}$, $V = \Delta^\mI$ and $E = \{\tuple{x,r,y}\!:\!p \mid$ $x,y \in V$, $p = r^\mI(x,y)\}$. Executing Algorithm~\ref{algCompCB} for $G$, the initialization sets $\bbP = \bbP_0 = \{\{a,a_2\},\{b,b_2,b_3,c,d,e\}\}$ and $\bbQ_r = \{V\}$. The first iteration of the ``while'' loop of Algorithm~\ref{algCompCB} uses $Y = V$ and $Y' = \{a,a_2\}$. It does not change $\bbP$, but refines $\bbQ_r$ to $\bbP$. Thus, the loop terminates after the first iteration. Let's denote $u = \{a, a_2\}$ and $v = \{b, b_2, b_3, c, d, e\}$. Then, the partition returned by Algorithm~\ref{algCompCB} is $\bbP = \{u, v\}$. The fuzzy interpretation $\mJ$ specified by the statement~\ref{step: JHDJH 3} of Algorithm~\ref{alg: minInt} has $\Delta^\mJ = \bbP$, $o^\mJ = u$ and $r^\mJ = \{\tuple{u,v}\!:\!0.8, \tuple{v,v}\!:1\}$. It is illustrated as $\mJ_1$ in the bottom left corner of Figure~\ref{fig: GDWIS}.

\item Case $\Phi = \Psi \cup \{O\}$: Consider the execution of Algorithm~\ref{alg: minInt} for $\mI$ and $\Phi$. Let $G = \tuple{V$, $E$, $\Label$, $\SV$, $\SE}$ be the fuzzy graph constructed by the statement~\ref{step: JHDJH 1}. We have $\SV = \{o\}$, $\SE = \{r\}$, $V = \Delta^\mI$, $E = \{\tuple{x,r,y}\!:\!p \mid$ $x,y \in V$, $p = r^\mI(x,y)\}$, $\Label(a)(o) = 1$ and $\Label(x)(o) = 0$ for $x \in V \setminus \{a\}$. Executing Algorithm~\ref{algCompCB} for $G$, the initialization sets $\bbP = \bbP_0 = \{\{a\},\{a_2\},\{b,b_2,b_3,c,d,e\}\}$ and $\bbQ_r = \{V\}$. The first iteration of the ``while'' loop of Algorithm~\ref{algCompCB} uses $Y = V$ and either $\{a\}$ or $\{a_2\}$ for $Y'$. Assume that it uses $Y' = \{a\}$. Thus, it does not change $\bbP$, but refines $\bbQ_r$ to $\{\{a\},\{a_2,b,b_2,b_3,c,d,e\}\}$. The second iteration of the ``while'' loop of Algorithm~\ref{algCompCB} uses $Y = \{a_2,b,b_2,b_3,c,d,e\}$ and $Y' = \{a_2\}$. It does not change~$\bbP$, but refines $\bbQ_r$ to $\bbP$. Thus, the loop terminates after the second iteration. Let's denote $u = \{a\}$, $u_2 = \{a_2\}$ and $v = \{b, b_2, b_3, c, d, e\}$. Then, the partition returned by Algorithm~\ref{algCompCB} is $\bbP = \{u, u_2, v\}$. The fuzzy interpretation $\mJ$ specified by the statement~\ref{step: JHDJH 3} of Algorithm~\ref{alg: minInt} has $\Delta^\mJ = \bbP$, $o^\mJ = u$ and $r^\mJ = \{\tuple{u,v}\!:\!0.8, \tuple{u_2,v}\!:\!0.8, \tuple{v,v}\!:1\}$. It is illustrated as $\mJ_2$ in the bottom right corner of Figure~\ref{fig: GDWIS}.

\item Case $\Phi = \Psi \cup \{I\}$: Consider the execution of Algorithm~\ref{alg: minInt} for $\mI$ and $\Phi$. Let $G = \tuple{V$, $E$, $\Label$, $\SV$, $\SE}$ be the fuzzy graph constructed by the statement~\ref{step: JHDJH 1}. We have $\SV = \emptyset$, $\SE = \{r,\cnv{r}\}$, $V = \Delta^\mI$ and $E = \{\tuple{x,r,y}\!:\!p$, $\tuple{y,\cnv{r},x}\!:\!p \mid$ $x,y \in V$, $p = r^\mI(x,y)\}$. Executing Algorithm~\ref{algCompCB} for $G$, the initialization sets $\bbP = \bbP_0 = \{\{a, a_2\}, \{b\}, \{c\}, \{b_2, b_3, d, e\}\}$ and $\bbQ_r = \bbQ_{\cnv{r}} = \{V\}$. One of the possible runs of the ``while'' loop of Algorithm~\ref{algCompCB} has subsequent iterations with the effects described below.
	\begin{enumerate}
	\item $\bbP$ does not change by splitting w.r.t.\ $\tuple{Y',Y,r}$, where $Y = V$ and $Y' = \{a, a_2\}$. 
	$\bbQ_r$ is refined to $\{\{a$, $a_2\}$, $\{b$, $b_2$, $b_3$, $c$, $d$, $e\}\}$.

	\item $\bbP$ is split w.r.t.\ $\tuple{Y',Y,r}$, where $Y = \{b$, $b_2$, $b_3$, $c$, $d$, $e\}$ and $Y' = \{b\}$, into $\{\{a\}$, $\{a_2\}$, $\{b\}$, $\{b_2$, $b_3$, $d$, $e\}$, $\{c\}\}$. $\bbQ_r$ is refined to $\{\{a$, $a_2\}$, $\{b\}$, $\{b_2$, $b_3$, $c$, $d$, $e\}\}$.
	
	\item $\bbP$ does not change by splitting w.r.t.\ $\tuple{Y',Y,r}$, where $Y = \{a,a_2\}$ and $Y' = \{a\}$. $\bbQ_r$ is refined to $\{\{a\}$, $\{a_2\}$, $\{b\}$, $\{b_2$, $b_3$, $c$, $d$, $e\}\}$.

	\item $\bbP$ does not change by splitting w.r.t.\ $\tuple{Y',Y,r}$, where $Y = \{b_2$, $b_3$, $c$, $d$, $e\}$ and $Y' = \{c\}$. $\bbQ_r$ is refined to $\{\{a\}$, $\{a_2\}$, $\{b\}$, $\{b_2$, $b_3$, $d$, $e\}$, $\{c\}\}$.

	\item $\bbP$ is split w.r.t.\ $\tuple{Y',Y,\cnv{r}}$, where $Y = V$ and $Y' = \{a_2\}$, into $\{\{a\}$, $\{a_2\}$, $\{b\}$, $\{b_2$, $b_3\}$, $\{c\}$, $\{d$, $e\}\}$. $\bbQ_{\cnv{r}}$ is refined to $\{\{a$, $b$, $b_2$, $b_3$, $c$, $d$, $e\}$, $\{a_2\}\}$.

	\item $\bbP$ is split w.r.t.\ $\tuple{Y',Y,\cnv{r}}$, where $Y = \{a$, $b$, $b_2$, $b_3$, $c$, $d$, $e\}$ and $Y' = \{b\}$, into $\{\{a\}$, $\{a_2\}$, $\{b\}$, $\{b_2$, $b_3\}$, $\{c\}$, $\{d\}$, $\{e\}\}$. $\bbQ_{\cnv{r}}$ is refined to $\{\{a$, $b_2$, $b_3$, $c$, $d$, $e\}$, $\{a_2\}$, $\{b\}\}$.

	\item $\bbP$ does not change by splitting w.r.t.\ $\tuple{Y',Y,\cnv{r}}$, where $Y = \{a$, $b_2$, $b_3$, $c$, $d$, $e\}$ and $Y' = \{c\}$. $\bbQ_{\cnv{r}}$ is refined to $\{\{a$, $b_2$, $b_3$, $d$, $e\}$, $\{a_2\}$, $\{b\}$, $\{c\}\}$.

	\item $\bbP$ does not change by splitting w.r.t.\ $\tuple{Y',Y,\cnv{r}}$, where $Y = \{a$, $b_2$, $b_3$, $d$, $e\}$ and $Y' = \{a\}$. $\bbQ_{\cnv{r}}$ is refined to $\{\{a\}$, $\{a_2\}$, $\{b\}$, $\{b_2$, $b_3$, $d$, $e\}$, $\{c\}\}$.

	\item $\bbP$ does not change by splitting w.r.t.\ $\tuple{Y',Y,\cnv{r}}$, where $Y = \{b_2$, $b_3$, $d$, $e\}$ and $Y' = \{e\}$. $\bbQ_{\cnv{r}}$ is refined to $\{\{a\}$, $\{a_2\}$, $\{b\}$, $\{b_2$, $b_3$, $d\}$, $\{c\}$, $\{e\}\}$.

	\item $\bbP$ does not change by splitting w.r.t.\ $\tuple{Y',Y,\cnv{r}}$, where $Y = \{b_2$, $b_3$, $d\}$ and $Y' = \{d\}$. $\bbQ_{\cnv{r}}$ is refined to $\{\{a\}$, $\{a_2\}$, $\{b\}$, $\{b_2$, $b_3\}$, $\{c\}$, $\{d\}$, $\{e\}\}$.

	\item $\bbP$ does not change by splitting w.r.t.\ $\tuple{Y',Y,r}$, where $Y = \{b_2$, $b_3$, $d$, $e\}$ and $Y' = \{e\}$. $\bbQ_r$ is refined to $\{\{a\}$, $\{a_2\}$, $\{b\}$, $\{c\}$, $\{b_2$, $b_3$, $d\}$, $\{e\}\}$.

	\item $\bbP$ does not change by splitting w.r.t.\ $\tuple{Y',Y,r}$, where $Y = \{b_2$, $b_3$, $d\}$ and $Y' = \{d\}$. $\bbQ_r$ is refined to $\{\{a\}$, $\{a_2\}$, $\{b\}$, $\{c\}$, $\{b_2$, $b_3\}$, $\{d\}$, $\{e\}\}$.
	\end{enumerate}
After the 12th iteration, we have $\bbQ_r = \bbQ_{\cnv{r}} = \bbP$ and the ``while'' loop of Algorithm~\ref{algCompCB} terminates. The partition returned by Algorithm~\ref{algCompCB} is $\bbP = \{\{a\}$, $\{a_2\}$, $\{b\}$, $\{c\}$, $\{b_2$, $b_3\}$, $\{d\}$, $\{e\}\}$. 
The fuzzy interpretation $\mJ$ specified by the statement~\ref{step: JHDJH 3} of Algorithm~\ref{alg: minInt} has $\Delta^\mJ = \bbP$, $o^\mJ = \{a\}$ and 
\begin{eqnarray*}
r^\mJ & = & \{\tuple{\{x\},\{y\}}\!:\!p \mid x,y \in \{a,b,c,d,e\},\ p = r^\mI(x,y) > 0\} \cup \\ 
      &   &\{\tuple{\{a_2\},\{b_2,b_3\}}\!:\!0.8, \tuple{\{b_2,b_3\},\{b_2,b_3\}}\!:\!1\}.
\end{eqnarray*}
It is illustrated as $\mJ_3$ in the top right corner of Figure~\ref{fig: GDWIS}.

\item Case $\Phi = \Psi \cup \{I,O\}$: Executing Algorithm~\ref{alg: minInt} for $\mI$ and $\Phi$ results in the same partition as in the above case, which is illustrated as $\mJ_3$ in the top right corner of Figure~\ref{fig: GDWIS}.
\myend
\end{itemize}
\end{example}

\begin{lemma}\label{lemma: HDHAO}
Let $G = \tuple{V, E, \Label, \SV, \SE}$ be the fuzzy graph constructed during the run of Algorithm~\ref{alg: minInt} for a finite fuzzy interpretation $\mI$ and $\{\triangle\} \subseteq \Phi \subseteq \{\triangle, \circ, \mor_r, *, ?, I, U, O\}$. Then, a reflexive binary relation $Z \subseteq \Delta^\mI \times \Delta^\mI$ is a $\Phi$-auto-bisimulation of $\mI$ iff it is a bisimulation of~$G$.
\end{lemma}

\begin{proof}
Let $Z$ be a reflexive binary relation on $\Delta^\mI$.

Consider the left-to-right implication and suppose $Z$ is a $\Phi$-auto-bisimulation of $\mI$. We prove that $Z$ is a bisimulation of~$G$ by showing that it satisfies Conditions~\eqref{eq: CB 1}-\eqref{eq: CB 3}. 
Since $Z$ is a $\Phi$-auto-bisimulation of $\mI$, it satisfies Conditions~\eqref{eq: FB 2}-\eqref{eq: FS 3b} and \eqref{eq: FS 4bis} (when $O \in \Phi$) for $\mI' = \mI$. 
Condition~\eqref{eq: CB 2} directly follows from Condition~\eqref{eq: FS 3} for $\mI' = \mI$, while Condition~\eqref{eq: CB 3} directly follows from Condition~\eqref{eq: FS 3b} for $\mI' = \mI$. 
Consider Condition~\eqref{eq: CB 1} and suppose $Z(x,x')$ holds. By Conditions~\eqref{eq: FB 2} and~\eqref{eq: FS 4bis} for $\mI' = \mI$, $A^\mI(x) = A^\mI(x')$ for all $A \in \CN$, and when $O \in \Phi$, $x = a^\mI$ iff $x' = a^\mI$, for all $a \in \IN$. Therefore, $\Label(x) = \Label(x')$ and Condition~\eqref{eq: CB 1} holds.  

Consider the right-to-left implication and suppose $Z$ is a bisimulation of $G$. We prove that $Z$ is a $\Phi$-auto-bisimulation of $\mI$ by showing that it satisfies Conditions \eqref{eq: FB 2}-\eqref{eq: FS 6b} for $\mI' = \mI$ (in the corresponding cases specified by $\Phi$). Since $Z$ is a bisimulation of $G$, it satisfies Conditions~\eqref{eq: CB 1}-\eqref{eq: CB 3} for all $r \in \SE$, which can be $\cnv{s}$ for $s \in \RN$ when $I \in \Phi$. Conditions~\eqref{eq: FB 2}-\eqref{eq: FS 3b} for $\mI' = \mI$ directly follow from Conditions~\eqref{eq: CB 1}-\eqref{eq: CB 3}, respectively. 
When $O \in \Phi$, Condition~\eqref{eq: FS 4bis} for $\mI' = \mI$ directly follows from Condition~\eqref{eq: CB 1}. 
When $U \in \Phi$, Conditions~\eqref{eq: FS 6} and~\eqref{eq: FS 6b} for $\mI' = \mI$ directly follow from the assumption that $Z$ is reflexive.
\myend
\end{proof}

\begin{theorem}
Algorithm~\ref{alg: minInt} is correct and can be implemented to run in time of the order $O((m \log{l} + n) \log{n})$ when zero-instances of roles are not explicitly stored, where $n = |\Delta^\mI|$, $m$ is number of nonzero instances of atomic roles of $\mI$ and $l$ is the number of different fuzzy values used for instances of atomic roles of $\mI$. The sizes of $\CN$, $\RN$ and $\IN$ are assumed to be constants. 
\end{theorem}

\begin{proof}
Let $\mI$ be a finite fuzzy interpretation and let $\{\triangle\} \subseteq \Phi \subseteq \{\triangle, \circ, \mor_r, *, ?, I, U, O\}$. Let $\mJ$ be the fuzzy interpretation returned by Algorithm~\ref{alg: minInt} for $\mI$ and $\Phi$. We prove that $\mJ = \mIsimP$. Let $G$ be the fuzzy graph and $\bbP$ the partition (of $\Delta^\mI$) constructed during the run of the algorithm. By the construction of $\mJ$ and $\mIsimP$, it is sufficient to prove that $\bbP$ is the partition that corresponds to the equivalence relation $\simPI$. Since Algorithm~\ref{algCompCB} is correct \cite[Theorem~3.6]{abs-2010-15671}, $\bbP$ is the partition that corresponds to the largest bisimulation of~$G$. By Lemma~\ref{lemma: HDHAO}, it follows that $\bbP$ is the partition that corresponds to the equivalence relation $\simPI$ (the largest $\Phi$-auto-bisimulation of $\mI$).

Let $n$, $m$ and $l$ be the values mentioned in the theorem. Clearly, $n$ is the number of vertices of $G$, $m$ or $2m$ is the number of nonzero edges of $G$ (depending on whether $I \in \Phi$), and $l$ is the number of different fuzzy degrees of edges of~$G$. Furthermore, the construction of $G$ from $\mI$ can be done in time $O(m+n)$, when zero-edges of $G$ are not explicitly stored.  Consider the step~\ref{step: minInt 2} of Algorithm~\ref{alg: minInt}. As mentioned before, the work~\cite{abs-2010-15671} has specified how to implement Algorithm~\ref{algCompCB} so that its time complexity is $O((m \log{l} + n) \log{n})$. The implementation uses data structures that allow us to get $\sup E(x,r,Y)$ in time $O(\log{l})$ for every $x, y \in V$ and $r \in \SE$ with $E(x,r,y) > 0$, where $Y$ is the block of $\bbP$ with $y \in Y$. The data structures also allow us to get in constant time the block of the resulting partition that contains a given vertex $x \in V$, as well as an element of a given block of the resulting partition. These facts imply that, when zero-instances of roles are not explicitly stored, the construction of $\mJ$ from $\bbP$ and the related data structures can be done in time $O(n + m\log{l})$. Summing up, Algorithm~\ref{alg: minInt} can be implemented to run in time $O((m \log{l} + n) \log{n})$.
\myend
\end{proof}

We have implemented Algorithm~\ref{alg: minInt} in Python~\cite{Minimization2023Code} appropriately so that its complexity order is \mbox{$O((m \log{l} + n) \log{n})$}. We have tested the correctness of the obtained program by implementing also a naive algorithm for the same problem and comparing their runs on many randomly generated inputs. Examples like~\ref{example: HDJQK3} and~\ref{example: IUFDJ} can be checked by using the program~\cite{Minimization2023Code} with the options ``verbose'' and ``with\_compound\_names''.  

%===============================================================================

\section{Conclusions}
\label{sec: conc}

The results of~\cite[Section~6]{FSS2020} on minimizing finite fuzzy interpretations in FDLs are formulated and proved only for the case of using the G\"odel family of fuzzy operators on the unit interval $[0,1]$ together with involutive negation. Besides, that work does not provide any algorithm for the minimization task. In this article, we have generalized the results of~\cite[Section~6]{FSS2020} for a much larger class of FDLs. 
Namely, the considered \FDLs use the Baaz projection operator instead of involutive negation and their semantics is specified by using an abstract algebra of fuzzy truth values, which can be any linear and complete residuated lattice. 
Furthermore, we have provided an efficient algorithm with a complexity of $O((m \log{l} + n) \log{n})$ for computing the quotient fuzzy interpretation $\mIsimP$ from a given finite fuzzy interpretation $\mI$, where $n$ is the size of the domain of $\mI$, $m$ is number of nonzero instances of atomic roles of $\mI$ and $l$ is the number of different fuzzy values used for instances of atomic roles of~$\mI$. 
We have proved that $\mIsimP$ is minimal among the ones that preserve fuzzy TBoxes and ABoxes under certain conditions.

The considered class of FDLs is rich, as it allows various role constructors as well as nominals (a concept constructor). However, to increase the simplicity, we did not consider (qualified/unqualified) number restrictions nor the concept constructor $\E r.\Self$. 
In~\cite{BSDL-INS,thesis-ARD}, Divroodi and Nguyen studied the problem of minimizing crisp interpretations in description logics also for the cases with number restrictions and/or the concept constructor $\E r.\Self$. They introduced the notion of a QS-interpretation that allows ``multi-edges'' and keeps information about ``self-edges'' (where ``edge'' is understood as an instance of a role). Minimizing fuzzy interpretations in FDLs with number restrictions and/or the concept constructor $\E r.\Self$ can be done by using fuzzy QS-interpretations and exploiting the results of the recent works~\cite{TFS2022,abs-2010-15671}. 

We have focused on minimizing a finite fuzzy interpretation in a FDL with the Baaz projection operator by using the largest crisp auto-bisimulation. The problem of minimizing a finite fuzzy interpretation in a FDL without the Baaz projection operator (and involutive negation) by using the greatest fuzzy auto-bisimulation, as in~\cite{minimization-by-fBS} but for a more general semantics (e.g., using any linear and complex residuated lattice), is more challenging and remains as future work.

%===============================================================================

\bibliography{BSfDL}
\bibliographystyle{elsarticle-num}

%===============================================================================

\end{document}